\documentclass{article}

\usepackage{microtype}
\usepackage{graphicx}
\usepackage{subfigure}
\usepackage{booktabs} 

\usepackage{hyperref}
\usepackage{balance}

\usepackage[accepted]{icml2025}

\usepackage{amsmath}
\usepackage{amssymb}
\usepackage{mathtools}
\usepackage{amsthm}
\usepackage{url}
\usepackage{makecell, multirow, tabularx}
\usepackage{wrapfig}
\usepackage{xcolor}
\usepackage{enumitem}
\usepackage{caption}  
\usepackage{subcaption} 

\usepackage[capitalize,noabbrev]{cleveref}

\theoremstyle{plain}
\newtheorem{theorem}{Theorem}[section]
\newtheorem{proposition}[theorem]{Proposition}

\theoremstyle{definition}
\newtheorem{definition}[theorem]{Definition}

\theoremstyle{remark}

\usepackage[textsize=tiny]{todonotes}

\icmltitlerunning{Neural-Symbolic Message Passing with Dynamic Pruning}

\begin{document}

\twocolumn[
\icmltitle{Neural-Symbolic Message Passing with Dynamic Pruning}




\begin{icmlauthorlist}
\icmlauthor{Chongzhi Zhang}{scut}
\icmlauthor{Junhao Zheng}{scut}
\icmlauthor{Zhiping Peng}{gupt,jp}
\icmlauthor{Qianli Ma}{scut}
\end{icmlauthorlist}

\icmlaffiliation{scut}{South China University of Technology, Guangzhou, China}
\icmlaffiliation{gupt}{Guangdong University of Petrochemical Technology, Maoming, China}
\icmlaffiliation{jp}{Jiangmen Polytechnic, Jiangmen, China}

\icmlcorrespondingauthor{Qianli Ma}{qianlima@scut.edu.cn}

\icmlkeywords{knowledge graph, complex query answering, neural and symbolic, message passing}

\vskip 0.3in
]



\printAffiliationsAndNotice{}  

\begin{abstract}
Complex Query Answering (CQA) over incomplete Knowledge Graphs (KGs) is a challenging task. Recently, a line of message-passing-based research has been proposed to solve CQA. However, they perform unsatisfactorily on negative queries and fail to address the noisy messages between variable nodes in the query graph. Moreover, they offer little interpretability and require complex query data and resource-intensive training. In this paper, we propose a Neural-Symbolic Message Passing (NSMP) framework based on pre-trained neural link predictors. By introducing symbolic reasoning and fuzzy logic, NSMP can generalize to arbitrary existential first order logic queries without requiring training while providing interpretable answers. Furthermore, we introduce a dynamic pruning strategy to filter out noisy messages between variable nodes. Experimental results show that NSMP achieves a strong performance. Additionally, through complexity analysis and empirical verification, we demonstrate the superiority of NSMP in inference time over the current state-of-the-art neural-symbolic method. Compared to this approach, NSMP demonstrates faster inference times across all query types on benchmark datasets, with speedup ranging from 2$\times$ to over 150$\times$. 
\end{abstract}

\section{Introduction}
\label{Introduction}

\begin{figure}
    \centering
    \includegraphics[width=0.5\linewidth]{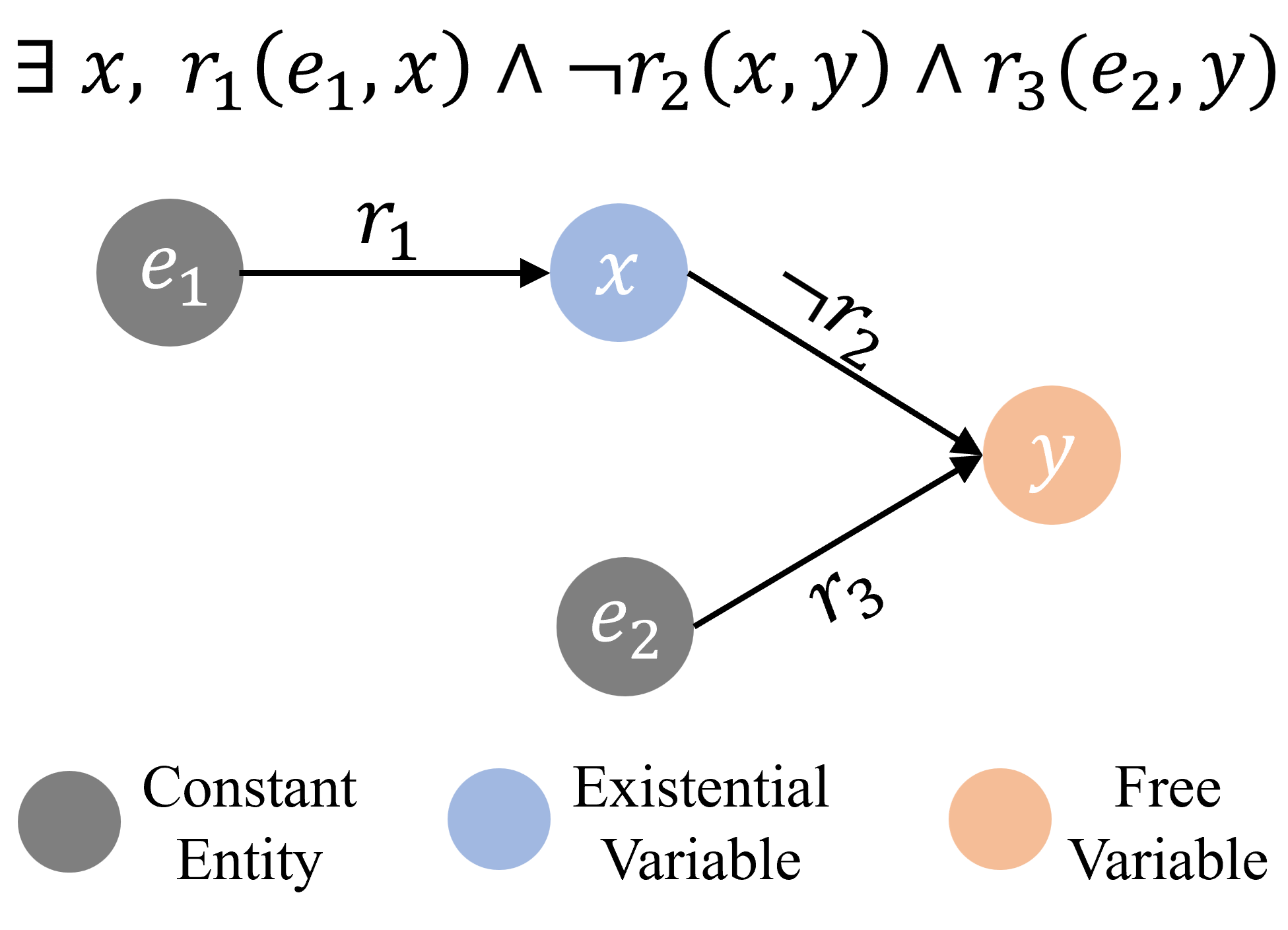}
    \caption{A query graph representation of a given logical formula. }
    \label{figure1}
\end{figure}

Knowledge graphs (KGs), which encode knowledge as relationships between entities, have received widespread attention from both academia and industry \citep{ji2021survey}. 
Complex Query Answering (CQA) over KGs is a fundamental and practical task, which requires answering existential first order logic formulas with logical operators, including conjunction ($\wedge$), disjunction ($\vee$), negation ($\neg$), and existential quantifier ($\exists$). 
A straightforward way is to traverse the KG to identify the answers directly \citep{zou2011gstore}. 
However, KGs often suffer from incompleteness due to the Open World Assumption (OWA) \citep{libkin2009open}. 
This makes it impossible to answer a complex query with missing links using traditional traversal methods. 
Therefore, CQA models need to possess the ability for non-trivial reasoning \citep{ren2020query2box, ren2020beta}.

Inspired by the success of neural link predictors \citep{bordes2013translating, trouillon2016complex,sun2018rotate, li2022house} on answering one-hop atomic queries on incomplete KGs, neural models \citep{hamilton2018embedding, ren2020beta, zhang2021cone,chen2022fuzzy,wang2023wasserstein,zhang2024pathformer} have been proposed to represent the entity sets by low-dimensional embeddings. 
Building on the foundation laid by these neural CQA models, message-passing-based research \citep{wang2023logical, zhang2024conditional} has demonstrated promising advancements in CQA. These message passing approaches represent logical formulas as query graphs, where edges correspond to atomic formulas (or their negations) with binary predicates, and nodes represent either constant entities or variables, as illustrated in Fig. \ref{figure1}. 
By utilizing pre-trained neural link predictors, they perform one-hop inference on edges, thereby inferring intermediate embeddings for variable nodes. An intermediate embedding is interpreted as a logical message passed from the neighboring node on the corresponding edge. Following the message passing paradigm \citep{gilmer2017neural}, the embeddings of variable nodes are updated to retrieve answers. Due to the integration of pre-trained neural link predictors, these message passing approaches are effective on both one-hop atomic and multi-hop complex queries. However, limitations still exist.

Firstly, even when augmented with fuzzy logic \citep{hajek2013metamathematics} for one-hop inference on negated edges, existing message passing CQA models still perform unsatisfactorily on negative queries. 
Secondly, while recent work \citep{zhang2024conditional} has considered noisy messages between variable and constant nodes, noisy messages between variable nodes remain unexplored. At the initial layers of message passing, messages inferred from neighboring variable nodes—whose states have not yet been updated—are ineffective. Aggregating such messages to update node states is essentially equivalent to introducing noise. 
Thirdly, similar to most neural models, they offer little interpretability and require training on large complex query datasets, which entails substantial training costs. In practical scenarios, gathering meaningful complex query data poses a significant challenge.

In this paper, we propose a Neural-Symbolic Message Passing (NSMP) framework, which leverages a simple pre-trained neural link predictor without requiring training on any complex query data. Specifically, 
NSMP conducts one-hop inference on edges by integrating neural and symbolic reasoning to compute intermediate states for variable nodes. 
The intermediate state can be interpreted as a message, represented by a fuzzy vector that denotes the fuzzy set of the variable within the corresponding atomic formula (i.e., the edge). 
In particular, we propose a novel pruning strategy that dynamically filters out unnecessary noisy messages between variable nodes during message passing. 
Based on fuzzy logic theory, NSMP aggregates the messages received by the variable nodes and updates the node states. 
Such a mechanism attains interpretability for the variables in the formula and can naturally execute the negation operator through fuzzy logic. Extensive experimental results show that NSMP achieves competitive performance with more efficient inference. 
In general, our main contributions can be summarized as follows: 
\begin{itemize}
    \item 
    We propose a neural-symbolic message passing framework that, for the first time, integrates neural and symbolic reasoning within a message passing CQA model. 
    This enables the message passing approach to naturally answer arbitrary existential first order logic formulas, similar to some other step-by-step neural-symbolic methods, without requiring training on complex queries, while also providing interpretability. 
    \item 
    We propose a novel dynamic pruning strategy to filter out noisy messages between variable nodes during message passing, thereby improving performance and efficiency. 
    \item 
    Extensive experimental results on benchmark datasets show that NSMP achieves a strong performance. In particular, NSMP significantly improves the performance of message passing CQA models on negative queries by introducing symbolic reasoning. 
    \item 
    Through computational complexity analysis, we reveal that our message-passing-based method offers more efficient inference than the current State-Of-The-Art (SOTA) step-by-step neural-symbolic method. Empirically, we verify that NSMP achieves faster inference times than the SOTA model, with speedups ranging from 2$\times$ to over 150$\times$.

\end{itemize}

\section{Related Work}
\label{Related Work}

In recent years, neural models \citep{hamilton2018embedding, ren2020beta} have been proposed to solve complex query answering by representing sets of entities using low-dimensional embeddings. Among these models, message-passing-based approaches \citep{wang2023logical,zhang2024conditional} have demonstrated promising potential. 
There are also several neural models enhanced with symbolic reasoning \citep{bai2023answering, yin2024rethinking}, which are capable of providing interpretable answers to complex queries. 
Further discussion for related work can be found in App. \ref{further related work}.

\section{Preliminaries}
\label{Preliminaries}

\paragraph{Knowledge Graphs and Complex Queries}

Let $\mathcal{V}$ be the finite set of entities and $\mathcal{R}$ be the finite set of relations. A knowledge graph $\mathcal{KG}$ can be defined as a set of triples $\mathcal{E} =\{(e_{h_{i}},r_{i},e_{t_{i}})\}$ that encapsulates the factual knowledge, where each triple encodes a relationship of type $r_{i} \in \mathcal{R}$ between the head and tail entity $e_{h_{i}}, e_{t_{i}} \in \mathcal{V}$. The complex queries that existing studies aim to address are essentially Existential First Order queries with a single free variable ($\text{EFO}_1$) \citep{yin2024rethinking}. In this case, a knowledge graph is a knowledge base defined by a specific first-order language \citep{wang2022logical}, where each entity $e\in \mathcal{V}$ is a constant, and each relation $r \in \mathcal{R}$ is a binary predicate that defines whether there is a directed edge between a pair of constant entities, i.e., $(e_i,r_{i},e_j) \in \mathcal{E}$ iff $r_i(e_i,e_j)=True$. We follow previous works \citep{ren2020beta,wang2023logical} and give the definition of  $\text{EFO}_1$ queries in Disjunctive Normal Form (DNF). 
\begin{definition}
\textnormal{($\text{EFO}_1$ Query in DNF)}.  The disjunctive normal form of an $\text{EFO}_1$ query $\phi$ is
\begin{equation}
\begin{split}
    \phi(y,x_{1},...,x_{k}) =  SCQ_1(y,x_{1},...,x_{k}) \\
    \vee... \vee  SCQ_d(y,x_{1},...,x_{k}), \label{first-order formula}
\end{split}
\end{equation}
where $y$ is a single free variable, $x_i$ represents the existential variables, $SCQ_i$ is $i$-th Sub-Conjunctive Query (SCQ) of $\phi(y,x_{1},...,x_{k})$, that is, $SCQ_i=\exists x_{1},...,\exists x_{k} a_{1}^{i}\wedge ...\wedge a_{n_{i}}^{i}$. $a_{j}^{i}$ is an atomic formula or its negation, i.e., $a_{j}^{i}=r(t_1,t_2)$ or $a_{j}^{i}=\neg r(t_1,t_2)$, $t_i$ is a term that is either a variable or a constant.  \label{def:def5}
\end{definition}
According to \citep{ren2020query2box}, one can solve $\text{EFO}_1$ queries containing the disjunction operator by solving sub-conjunctive queries with DNF. This DNF-based approach has been shown to be effective and scalable in previous studies \citep{ren2022smore, wang2023logical}. 
For a complex query $\phi$, let $\mathcal{A}[\phi]$ represent the set of answer entities. In the case of using DNF, $\mathcal{A}[\phi]$ is the union of the answer sets for its sub-conjunctive queries, i.e., $\mathcal{A}[\phi] =  {\textstyle \bigcup_{i=1}^{d}} \mathcal{A}[SCQ_i]$. 
For the answer set of $SCQ_i$, an entity $e_a \in \mathcal{V}$ is considered an answer to $SCQ_i$ only when $SCQ_i(y=e_a,x_{1},...,x_{k})=True$. 
Clearly, in this case, we have $e_a \in \mathcal{A}[SCQ_i] \in \mathcal{A}[\phi]$. In our work, we follow previous works \citep{wang2023logical,zhang2024conditional} and use DNF to answer $\text{EFO}_1$ queries.

\paragraph{Operator Tree and Query Graph}

Most of the previous CQA models represent complex queries as operator trees \citep{ren2020query2box, wang2022benchmarking}. In this case, first-order logic operators are transformed into set logic operators. Specifically, conjunctions are replaced by intersections, disjunctions by unions, negations by complements, and the existential quantifier is replaced by set projection. However, according to \citep{yin2024rethinking}, operator trees can only handle cyclic $\text{EFO}_1$ queries in an approximate way. In contrast, more expressive query graphs \citep{wang2023logical} provide a representation that can naturally model conjunctive queries,  as illustrated in the example in Fig. \ref{figure1}. In this paper, we follow prior work \citep{wang2023logical} and represent conjunctive queries as expressive query graphs. 
Following \citep{yin2024rethinking}, we give the definition to describe the query graph here. 
\begin{definition}
\textnormal{(Query Graph)}. A query graph can be defined as a set of atomic formulas with a possible negation operator, i.e., $\{a_{i}\}$. According to Definition \ref{def:def5}, $a_i$ defines an edge from term node $t_1$ to term node $t_2$, where the type of the edge is either $r$ or $\neg r$. Specifically, this edge is determined by two terms, a directed edge of type $r$ and a negation indicator (positive or negative). \label{def:def6}
\end{definition}

\paragraph{Neural Link Predictors}

A neural link predictor \citep{trouillon2016complex} is a latent feature model used to answer one-hop atomic queries. It leverages its scoring function to compute a continuous truth value $\varphi(h,r,t) \in [0,1]$ for an input triple, where $h$, $r$, and $t$ represent the embeddings of the head entity, the relation, and the tail entity, respectively. Since the neural link predictor is applied within the message passing CQA models \citep{wang2023logical,zhang2024conditional} to perform one-hop inference on edges in a query graph, the input embeddings also involve variable embeddings.

\paragraph{Neural One-hop Inference on Edges}
\label{neural one-hop}

To utilize the neural link predictor for answering complex queries, \citet{wang2023logical} proposed performing neural one-hop inference on edges (i.e., atomic formulas or their negations) in the query graph to compute logical messages, and thus defining a neural message encoding function $\rho$ in the form of continuous truth value maximization. 
According to Definition \ref{def:def6}, an edge in the query graph is determined by the terms, directed relation type, and negation indicator, which together serve as the input arguments to $\rho$. 
Based on the direction information of the edge and the negation indicator, \citet{wang2023logical} divided the neural message encoding function 
$\rho$ into four cases. 
Firstly, for a non-negated edge, given the embedding of the tail node $t$ and the relation embedding $r$, $\rho$ is formulated to infer the intermediate embedding $\hat{h}$ for the head node on this edge:
\begin{equation}
\begin{split}
    \hat{h} =\rho (t,r,D_{t\to h},Pos):=\underset{v\in \mathcal{R}_s} {arg\ max}\ \varphi(v,r,t), \label{rho1}
\end{split}
\end{equation}
where $\mathcal{R}_s$ is the search domain for the variable embedding $v$, $D_{t\to h}$ indicates the direction information (i.e., inferring the message passed from the tail node to the head node), and $Pos$ denotes that this edge is a positive atomic formula. 
Accordingly, one can infer the intermediate embedding $\hat{t}$ for the tail node given head embedding $h$ and relation embedding $r$: 
\begin{equation}
\begin{split}
    \hat{t} =\rho (h,r,D_{h\to t},Pos):=\underset{v\in \mathcal{R}_s }{arg\ max}\ \varphi(h,r,v), \label{rho2}
\end{split}
\end{equation}
where $D_{h\to t}$ indicates inferring the message passed from the head node to the tail node. For negated edges, \citet{wang2023logical} proposed utilizing the fuzzy logic negator \citep{hajek2013metamathematics} to infer intermediate embeddings for nodes on the edges:
\begin{align}
    \hat{h} =\rho (t,r,D_{t\to h},Neg) &:= \underset{v\in \mathcal{R}_s }{arg\ max}\ \varphi(v,\neg r,t) \notag \\
    &= \underset{v\in \mathcal{R}_s }{arg\ max}\ [1 - \varphi(v,r,t)]
    , \label{rho3} \\
    \hat{t} =\rho (h,r,D_{h\to t},Neg) &:= \underset{v\in \mathcal{R}_s }{arg\ max}\ \varphi(h,\neg r,v) \notag  \\ 
    &= \underset{v\in \mathcal{R}_s }{arg\ max}\ [1 - \varphi(h,r,v)], \label{rho4} 
\end{align}
where $Neg$ denotes that this edge is a negated atomic formula. 
In our work, we perform symbolic-integrated one-hop inference on edges (see Sec. \ref{ns one-hop}), incorporating both neural and symbolic components into our message encoding function. 
We follow \citep{wang2023logical} and use the above neural message encoding function $\rho$ as the neural component of our neural-symbolic message encoding function.

\section{Proposed Method}

In this section, we first propose how to integrate symbolic reasoning into neural one-hop inference to compute neural-symbolic messages. Then, we propose our message passing mechanism based on fuzzy logic and dynamic pruning from two views. Finally, we analyze the computational complexity of the proposed method to reveal the superiority of our message-passing-based method in terms of efficiency. 

\subsection{Neural-Symbolic One-hop Inference on Edges}
\label{ns one-hop}

Since we aim to integrate symbolic information into neural message passing, there are two different representations of entities and relations in our work: neural and symbolic representations. For the neural representation, since we utilize the pre-trained link predictor to compute neural messages, the entities and relations are encoded into the embedding space of the pre-trained neural link predictor. 
That is, they already have pre-trained embeddings. 
For the symbolic representation, each entity $e \in \mathcal{V}$ is encoded as a one-hot vector $p_e \in \{0,1\}^{1\times |\mathcal{V}|}$ and each relation $r \in \mathcal{R}$ is represented as an adjacency matrix $M_r \in \{0,1\}^{|\mathcal{V}|\times |\mathcal{V}|}$, where 
$M_r^{ij}=1$ if $(e_i,r,e_j) \in \mathcal{E}$ else $M_r^{ij}=0$. 
We follow \citep{zhang2024conditional} and only consider inferring the intermediate state for the variable, so the neural and symbolic representations of entities and relations remain unchanged. 
Accordingly, we also assign neural and symbolic representations to the variable nodes in the query graph. 
Specifically, each variable is equipped with a corresponding embedding and symbolic vector.
However, the symbolic representation of a variable is not a one-hot vector but a fuzzy vector $p_v \in [0,1]^{1\times |\mathcal{V}|}$ that represents a fuzzy set. Each element of $p_v$ can be interpreted as the probability of the corresponding entity. 

\paragraph{Symbolic One-hop Inference on Edges}
To conduct symbolic reasoning on edges, we follow TensorLog \citep{cohen2020tensorlog} and define a symbolic one-hop inference function $\mu$ for four cases depending on the input arguments. 
The first situation is to infer the intermediate symbolic vector $\hat{p_t}$ of the tail node given the head symbolic vector $p_h$ and relational adjacency matrix $M_r$ on a non-negated edge: 
\begin{equation}
\begin{split}
    \hat{p_t}=\mu(p_h,M_r,D_{h\to t},Pos):=\mathcal{N}(p_hM_r)
    . \label{mu1}
\end{split}
\end{equation}
$\mathcal{N}(\cdot)$ is a thresholded normalized function that is defined as follows:
\begin{equation}
\begin{split}
    \mathcal{N}(p) = \frac{p \cdot \mathbf{1}(p \geq \epsilon )}{\max\left( \epsilon, \sum \left(p \cdot \mathbf{1}(p \geq \epsilon )\right)\right)}
    , \label{normalized function}
\end{split}
\end{equation}
where $\epsilon$ represents the threshold and $\mathbf{1}(p \geq \epsilon )$ is an indicator function. 
Similarly, the intermediate symbolic vector $\hat{p_h}$ of the head node can be inferred given the tail symbolic vector $p_t$ and $M_r$:
\begin{equation}
\begin{split}
    \hat{p_h}=\mu(p_t,M_r,D_{t\to h},Pos):=\mathcal{N}(p_tM_{r}^{\top})
    , \label{mu2}
\end{split}
\end{equation}
where $\top$ stands for transpose. 
Based on the fuzzy logic theory \citep{klement2013triangular, hajek2013metamathematics}, 
we follow \citep{xu2022neural} to introduce a hyperparameter $\alpha$ to the fuzzy logic negator and define the estimation of the intermediate symbolic vector on negated edges as follows:
\begin{align}
    \hat{p_t}=\mu(p_h,M_r,D_{h\to t},Neg) &:=\mathcal{N}(\frac{\alpha }{|\mathcal{V} |} - p_hM_{r}), \label{mu3} \\
    \hat{p_h}=\mu(p_t,M_r,D_{t\to h},Neg) &:=\mathcal{N}(\frac{\alpha }{|\mathcal{V} |} -p_tM_{r}^{\top})
    . \label{mu4}
\end{align}

\paragraph{Integrating Neural and Symbolic Reasoning}
In order to integrate neural reasoning to enhance symbolic reasoning, we follow \citep{xu2022neural} and convert the intermediate embedding obtained by the neural message encoding function $\rho$ into a fuzzy vector. 
Specifically, for an intermediate embedding inferred by $\rho$, we first compute its similarity with the embeddings of all entities. After applying a softmax operation, we can obtain a fuzzy vector $p' \in [0,1]^{1\times |\mathcal{V}|}$. We define this procedure as a function $f$ as follows:
\begin{equation}
\begin{split}
    f(\rho) = softmax ( \underset{\forall e\in \mathcal{V} }{concat}  ( \mathcal{S}(\rho, E_{e}) ) )
    , \label{f func}
\end{split}
\end{equation}
where $E_{e}$ represents the embedding of the entity $e$, $\mathcal{S}$ is a binary similarity function, and $concat$ is a function that maps the similarities between all entities and the intermediate embedding inferred by $\rho$ to a vector. 
Depending on the selected pre-trained neural link predictor, $\mathcal{S}$ can either be an inner-product-based or a distance-based scoring function. 

Then, we can define our neural-symbolic message encoding function $\varrho$, which also has four cases depending on the input arguments.  
\begin{figure*}
    \centering
    \includegraphics[width=1\textwidth]{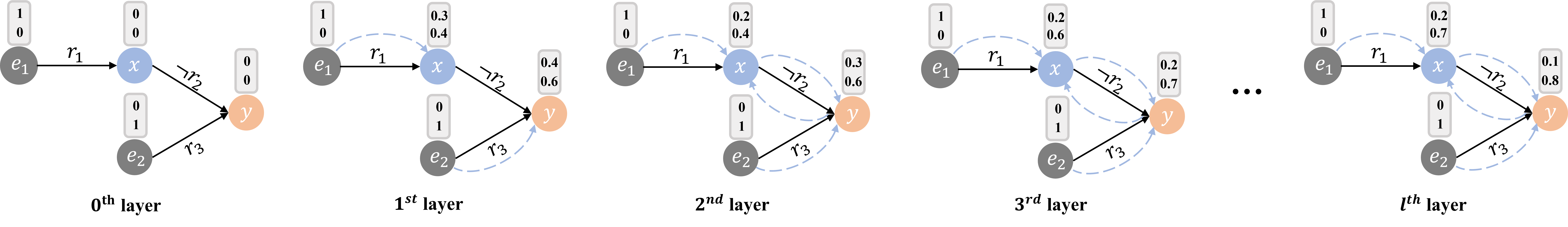}
    \caption{A toy example to show the process of dynamic pruning. The blue arrow represents the passing of the message computed by $\varrho$, and the vector indicates the state of the corresponding node after message passing at each layer. }
    \label{figure2}
\end{figure*}
Given the embedding $t$ and symbolic vector $p_t$ of the tail node on a non-negated edge, as well as the embedding $r$ and adjacency matrix $M_r$ of the relation, we infer the intermediate state $p_{\hat{h}}$ for the variable node at the head position on this edge. We formulate the inference task in the form of neural-enhanced symbolic reasoning: 
\begin{align}
    &p_{\hat{h}} = \varrho (t,p_t,r,M_r,D_{t\to h},Pos):= \notag \\
                & \mathcal{N}(f(\rho (t,r,D_{t\to h},Pos)) +\mu(p_t,M_r,D_{t\to h},Pos)). \label{varrho1}
\end{align}
This approach leverages the embeddings inferred by pre-trained neural link predictors to enhance symbolic reasoning results. It enables symbolic reasoning to handle missing links in observed knowledge graphs. Additionally, such an approach can represent the membership degree of variables concerning all entities in the form of fuzzy sets, thereby providing interpretability. 
Similarly, for the other three cases, the encoding functions $\varrho$ are as
follows: 
\begin{align}
    &p_{\hat{t}} =\varrho (h,p_h,r,M_r,D_{h\to t},Pos):= \notag \\
    &\mathcal{N}(f(\rho (h,r,D_{h\to t},Pos))
    +\mu(p_h,M_r,D_{h\to t},Pos)), \label{varrho2} \\
    &p_{\hat{h}} =\varrho (t,p_t,r,M_r,D_{t\to h},Neg):=  \notag \\
    &\mathcal{N}(f(\rho (t,r,D_{t\to h},Neg)) +\mu(p_t,M_r,D_{t\to h},Neg)), \label{varrho3} \\
    &p_{\hat{t}} =\varrho (h,p_h,r,M_r,D_{h\to t},Neg):= \notag \\
    &\mathcal{N}(f(\rho (h,r,D_{h\to t},Neg))+\mu(p_h,M_r,D_{h\to t},Neg)). \label{varrho4}
\end{align}

\subsection{Neural-Symbolic Message Passing}

In this subsection, we propose a Neural-Symbolic Message Passing (NSMP) framework. This framework builds on the neural-symbolic one-hop inference proposed in Sec. \ref{ns one-hop} and incorporates a dynamic pruning strategy. 
Since NSMP is a variant of message passing networks \citep{gilmer2017neural, xu2018how}, we focus on the details of the operations within a single NSMP layer. Specifically, each NSMP layer consists of two stages, each operating under different views: (1) Query Graph View: Passing neural-symbolic messages on the query graph with our proposed dynamic pruning strategy; (2) Node State View: Updating the state of each variable node that has received messages at the current layer. In the following, we first describe our framework from these two views and then explain how to answer complex queries with the proposed NSMP.

\subsubsection{Query Graph View: Message Passing with Dynamic Pruning}
\label{DP}

For constant nodes in a query graph, we follow \citep{zhang2024conditional} and do not pass messages to constant nodes. 
However, the situation varies across different layers of message passing when passing messages to variable nodes in a query graph. 
At the initial layers, messages passed from neighboring variable nodes, whose states have not yet been updated, can be regarded as noise, as these nodes do not carry any meaningful information at this stage. 
In contrast, at the later layers, messages from neighboring variable nodes with updated states provide valuable information derived from constants, which is especially important for variable nodes not directly connected to constant nodes. 
To dynamically filter out unnecessary noisy messages while retaining valuable ones, 
we determine if a variable node should pass messages to its neighboring variable nodes by checking whether its state has undergone any updates. 
Specifically, a variable node is allowed to pass messages to its neighboring variable nodes only when its state has been updated. 
We refer to such a strategy as dynamic pruning. The ``dynamic" nature lies in the fact that this strategy dynamically adapts the pruning process for different cases, without the need to design specific pruning strategies for different query graphs or layer indices. In particular, this strategy not only filters out noisy messages, but also effectively avoids the computation of noise messages to improve the efficiency of message passing. 
Fig. \ref{figure2} follows the query graph in Fig. \ref{figure1} to illustrate an example of this pruning strategy at different layers.

\subsubsection{Node State View: Neural-Symbolic State Update Scheme}
\label{Node State}

After message passing on a query graph, for each node that has received messages, we need to aggregate the received messages and update its intermediate state. Let $n_c^{(l)}$ and $s_c^{(l)}$ denote the neural and symbolic state of a constant node at the $l^{th}$ layer, respectively, while $n_v^{(l)}$ and $s_v^{(l)}$ represent the neural and symbolic state of a variable node at the same layer. According to Sec. \ref{DP}, we do not pass messages to constant nodes and variable nodes that do not receive messages. In this case, we do not update the state of these nodes at the corresponding NSMP layer. In particular, the state of a constant node at all layers is equal to its initial state, i.e., $n_{c}^{(0)}=n_{c}^{(l)}$ and $s_{c}^{(0)}=s_{c}^{(l)}$, where $n_{c}^{(0)}$ and $s_{c}^{(0)}$ represent the neural pre-trained embedding and the symbolic one-hot vector of the corresponding constant entity, respectively. For the state of a variable node that does not receive messages, its state at the current layer is equal to its state at the previous layer, i.e., $n_{v}^{(l)}=n_{v}^{(l-1)}$ and $s_{v}^{(l)}=s_{v}^{(l-1)}$, where $l>0$. Both $n_{v}^{(0)}$ and $s_{v}^{(0)}$ are vectors containing all zeros.

For a variable node that has received messages, the messages it receives from neighboring nodes essentially represent intermediate states inferred by the neural-symbolic message encoding function $\varrho$, as proposed in Section \ref{ns one-hop}. 
These intermediate states, which are fuzzy vectors representing fuzzy sets, can be aggregated using fuzzy logic theory to update the state of the variable node.
Specifically, for a variable node $v$, we form a neighbor set $Neigbor_{DP}(v)$ that includes all its neighboring variable nodes with updated states and all its neighboring constant nodes, in accordance with the dynamic pruning strategy. 
For each neighboring node $u \in Neigbor_{DP}(v)$, the neural-symbolic message $m^{(l)}$ passed to the variable node $v$ by $u$ at the $l^{th}$ layer can be computed using the encoding function $\varrho$, 
which depends on the information about the directed relation type and negation between the $u$ and $v$, as well as the neural and symbolic state of the node $u$ at the ${(l-1)}^{th}$ layer. 
After neural-symbolic message passing, we employ product fuzzy logic to aggregate the messages received by $v$ and update the symbolic state of $v$ at the $l^{th}$ layer: 
\begin{equation}
\begin{split}
    s_{v}^{(l)} = \mathcal{N}(m^{(l)}_{1}\circ \dots \circ m^{(l)}_{k_{v}}),
    \label{product conj}
\end{split}
\end{equation}
where $\circ$ is Hadmard product, $ m^{(l)}_{1},\dots,m^{(l)}_{k_{v}}$ denotes $k_{v}$ messages received by $v$, $k_{v} \ge 1$ represents the number of neighboring nodes in $Neigbor_{DP}(v)$. Each message $m_i^{(l)}$ can be computed with $\varrho$: 
\begin{equation}
\begin{split}
    m^{(l)}_{i} = \varrho (n_{u_i}^{(l-1)},s_{u_i}^{(l-1)},r_{u_{i}v},M_{r_{u_{i}v}},D_{u_{i}\to v},NI{u_{i}v}), \label{varrho-nv}
\end{split}
\end{equation}
where $D_{u_{i}\to v} \in \{D_{h\to t}, D_{t\to h}\}$ indicates the direction and $NI{u_{i}v}\in \{Pos, Neg\}$ is a negation indicator. 
While for the neural state $n_{v}^{(l)}$, we utilze the updated symbolic state $s_{v}^{(l)}$ to update it. 
Specifically, we form an entity set $\mathcal{V}_{nz}$ consisting of the entities corresponding to the non-zero elements in the fuzzy vector $s_{v}^{(l)}$. 
We then aggregate the embeddings of these entities weighted by their corresponding probabilities. 
\begin{equation}
\begin{split}
    n_{v}^{(l)}=\sum_{i=1}^{|\mathcal{V}_{nz}|} s_{v,i}^{(l)}E_{e_{i}},e_i \in \mathcal{V}_{nz}, \label{s2z}
\end{split}
\end{equation}
where $s_{v,i}^{(l)}$ is the corresponding probability of $e_i$ in $s_{v}^{(l)}$ and $E_{e_i}$ is the embedding of the entity $e_i$.

\subsubsection{How to Answer Complex Queries with Proposed NSMP}
\label{inference exp}

According to Definition \ref{def:def5}, a DNF query can be answered by solving all of its sub-conjunctive queries. 
For a given conjunctive query $Q$, we employ NSMP layers $L$ times to the query graph of $Q$, where $L$ is the depth of NSMP. 
Then, we use the symbolic state $s_{y}^{(L)}$ of the free variable node $y$ at the final layer, along with the corresponding neural state $n_{y}^{(L)}$, to obtain the probability of each entity being the answer, thereby retrieving the final answers: 
\begin{equation}
\begin{split}
    &p_{\mathcal{A[Q]}} =\lambda s_{y}^{(L)}\\ 
    &+ (1-\lambda)softmax ( \underset{\forall e\in \mathcal{V} }{concat}  ( cos(n_{y}^{(L)}, E_{e}) ) )
    , \label{ns_eva}
\end{split}
\end{equation}
where $\lambda$ is a hyperparameter that balances the influence of neural and symbolic representation and $cos(\cdot,\cdot)$ is the cosine similarity. 
Let $D$ denote the largest distance between the constant nodes and the free variable node.
According to \citep{wang2023logical}, $L$ should be larger than or equal to $D$ to ensure the free variable node successfully receives all messages from the constant nodes. 
Therefore, $L$ should dynamically change based on different types of conjunctive queries. 
In addition, since the pre-trained link predictor is frozen in our work, NSMP has no trainable parameters.

\begin{table*}[htbp]
  \centering
  \caption{MRR results of baselines and our model on BetaE datasets. The average score is calculated separately among positive and negative queries. Highlighted are the top \textcolor[rgb]{1, 0, 0}{\textbf{first}} and \textcolor[rgb]{0, 0, 1}{\textbf{second}} results.  }
  \resizebox{0.72\linewidth}{!}{
    \begin{tabular}{ccccccccccccccccc}
    \toprule
    \textbf{KG} & \textbf{Model} & \textbf{1p} & \textbf{2p} & \textbf{3p} & \textbf{2i} & \textbf{3i} & \textbf{pi} & \textbf{ip} & \textbf{2u} & \textbf{up} & \textbf{AVG.(P)} & \textbf{2in} & \textbf{3in} & \textbf{inp} & \textbf{pin} & \textbf{AVG.(N)} \\
    \midrule
    \multirow{10}[4]{*}{FB15k-237} & BetaE & 39.0  & 10.9  & 10.0  & 28.8  & 42.5  & 22.4  & 12.6  & 12.4  & 9.7   & 20.9  & 5.1   & 7.9   & 7.4   & 3.6   & 6.0 \\
          & CQD   & \textcolor[rgb]{1, 0, 0}{\textbf{46.7}} & 10.3  & 6.5   & 23.1  & 29.8  & 22.1  & 16.3  & 14.2  & 8.9   & 19.8  & 0.2   & 0.2   & 2.1   & 0.1   & 0.7 \\
          & FuzzQE & 42.8  & 12.9  & 10.3  & 33.3  & 46.9  & 26.9  & 17.8  & 14.6  & 10.3  & 24.0  & 8.5   & 11.6  & 7.8   & 5.2   & 8.3 \\
          & GNN-QE & 42.8  & \textcolor[rgb]{0, 0, 1}{\textbf{14.7}} & \textcolor[rgb]{0, 0, 1}{\textbf{11.8}} & \textcolor[rgb]{0, 0, 1}{\textbf{38.3}} & \textcolor[rgb]{1, 0, 0}{\textbf{54.1}} & \textcolor[rgb]{0, 0, 1}{\textbf{31.1}} & 18.9  & 16.2  & \textcolor[rgb]{1, 0, 0}{\textbf{13.4}} & \textcolor[rgb]{0, 0, 1}{\textbf{26.8}} & 10.0  & \textcolor[rgb]{0, 0, 1}{\textbf{16.8}} & \textcolor[rgb]{0, 0, 1}{\textbf{9.3}} & 7.2   & 10.8 \\
          & ENeSy & 44.7  & 11.7  & 8.6   & 34.8  & 50.4  & 27.6  & 19.7  & 14.2  & 8.4   & 24.5  & 10.1  & 10.4  & 7.6   & 6.1   & 8.6 \\
          & $\text{CQD}^\mathcal{A}$  & \textcolor[rgb]{1, 0, 0}{\textbf{46.7}} & 13.6  & 11.4  & 34.5  & 48.3  & 27.4  & \textcolor[rgb]{0, 0, 1}{\textbf{20.9}} & \textcolor[rgb]{1, 0, 0}{\textbf{17.6}} & 11.4  & 25.7  & \textcolor[rgb]{1, 0, 0}{\textbf{13.6}} & 16.8  & 7.9   & \textcolor[rgb]{1, 0, 0}{\textbf{8.9}} & \textcolor[rgb]{0, 0, 1}{\textbf{11.8}} \\
\cmidrule{2-17}          & \multicolumn{16}{l}{(Based on message passing)} \\
          & LMPNN & \textcolor[rgb]{0, 0, 1}{\textbf{45.9}} & 13.1  & 10.3  & 34.8  & 48.9  & 22.7  & 17.6  & 13.5  & 10.3  & 24.1  & 8.7   & 12.9  & 7.7   & 4.6   & 8.5 \\
          & CLMPT & 45.7  & 13.7  & 11.3  & 37.4  & 52.0  & 28.2  & 19.0  & 14.3  & 11.1  & 25.9  & 7.7   & 13.7  & 8.0   & 5.0   & 8.6 \\
          & NSMP  & \textcolor[rgb]{1, 0, 0}{\textbf{46.7}} & \textcolor[rgb]{1, 0, 0}{\textbf{15.1}} & \textcolor[rgb]{1, 0, 0}{\textbf{12.3}} & \textcolor[rgb]{1, 0, 0}{\textbf{38.7}} & \textcolor[rgb]{0, 0, 1}{\textbf{52.2}} & \textcolor[rgb]{1, 0, 0}{\textbf{31.2}} & \textcolor[rgb]{1, 0, 0}{\textbf{23.3}} & \textcolor[rgb]{0, 0, 1}{\textbf{17.2}} & \textcolor[rgb]{0, 0, 1}{\textbf{11.9}} & \textcolor[rgb]{1, 0, 0}{\textbf{27.6}} & \textcolor[rgb]{0, 0, 1}{\textbf{11.9}} & \textcolor[rgb]{1, 0, 0}{\textbf{17.6}} & \textcolor[rgb]{1, 0, 0}{\textbf{10.8}} & \textcolor[rgb]{0, 0, 1}{\textbf{7.9}} & \textcolor[rgb]{1, 0, 0}{\textbf{12.0}} \\
    \midrule
    \multirow{10}[4]{*}{NELL995} & BetaE & 53.0  & 13.0  & 11.4  & 37.6  & 47.5  & 24.1  & 14.3  & 12.2  & 8.5   & 24.6  & 5.1   & 7.8   & 10.0  & 3.1   & 6.5 \\
          & CQD   & \textcolor[rgb]{1, 0, 0}{\textbf{60.8}} & 18.3  & 13.2  & 36.5  & 43.0  & 30.0  & 22.5  & 17.6  & 13.7  & 28.4  & 0.1   & 0.1   & 4.0   & 0.0   & 1.1 \\
          & FuzzQE & 47.4  & 17.2  & 14.6  & 39.5  & 49.2  & 26.2  & 20.6  & 15.3  & 12.6  & 27.0  & 7.8   & 9.8   & 11.1  & 4.9   & 8.4 \\
          & GNN-QE & 53.3  & 18.9  & 14.9  & 42.4  & 52.5  & 30.8  & 18.9  & 15.9  & 12.6  & 28.9  & 9.9   & 14.6  & 11.4  & 6.3   & 10.6 \\
          & ENeSy & 59.0  & 18.0  & 14.0  & 39.6  & 49.8  & 29.8  & 24.8  & 16.4  & 13.1  & 29.4  & 11.3  & 8.5   & 11.6  & 8.6   & 10.0 \\
          & $\text{CQD}^\mathcal{A}$  & 60.4  & \textcolor[rgb]{1, 0, 0}{\textbf{22.9}} & \textcolor[rgb]{0, 0, 1}{\textbf{16.7}} & \textcolor[rgb]{0, 0, 1}{\textbf{43.4}} & \textcolor[rgb]{0, 0, 1}{\textbf{52.6}} & \textcolor[rgb]{0, 0, 1}{\textbf{32.1}} & \textcolor[rgb]{0, 0, 1}{\textbf{26.4}} & \textcolor[rgb]{1, 0, 0}{\textbf{20.0}} & \textcolor[rgb]{1, 0, 0}{\textbf{17.0}} & \textcolor[rgb]{0, 0, 1}{\textbf{32.3}} & \textcolor[rgb]{1, 0, 0}{\textbf{15.1}} & \textcolor[rgb]{1, 0, 0}{\textbf{18.6}} & \textcolor[rgb]{1, 0, 0}{\textbf{15.8}} & \textcolor[rgb]{1, 0, 0}{\textbf{10.7}} & \textcolor[rgb]{1, 0, 0}{\textbf{15.1}} \\
\cmidrule{2-17}          & \multicolumn{16}{l}{(Based on message passing)} \\
          & LMPNN & 60.6  & 22.1  & 17.5  & 40.1  & 50.3  & 28.4  & 24.9  & 17.2  & 15.7  & 30.7  & 8.5   & 10.8  & 12.2  & 3.9   & 8.9 \\
          & CLMPT & 58.9  & \textcolor[rgb]{0, 0, 1}{\textbf{22.1}} & \textcolor[rgb]{1, 0, 0}{\textbf{18.4}} & 41.8  & 51.9  & 28.8  & 24.4  & 18.6 & \textcolor[rgb]{0, 0, 1}{\textbf{16.2}} & 31.3  & 6.6   & 8.1   & 11.8  & 3.8   & 7.6 \\
          & NSMP  & \textcolor[rgb]{0, 0, 1}{\textbf{60.7}} & 21.6  & 17.5  & \textcolor[rgb]{1, 0, 0}{\textbf{44.2}} & \textcolor[rgb]{1, 0, 0}{\textbf{53.8}} & \textcolor[rgb]{1, 0, 0}{\textbf{33.7}} & \textcolor[rgb]{1, 0, 0}{\textbf{26.7}} & \textcolor[rgb]{0, 0, 1}{\textbf{19.1}}  & 14.4  & \textcolor[rgb]{1, 0, 0}{\textbf{32.4}} & \textcolor[rgb]{0, 0, 1}{\textbf{12.4}} & \textcolor[rgb]{0, 0, 1}{\textbf{15.7}} & \textcolor[rgb]{0, 0, 1}{\textbf{13.7}} & \textcolor[rgb]{0, 0, 1}{\textbf{7.8}} & \textcolor[rgb]{0, 0, 1}{\textbf{12.4}} \\
    \bottomrule
    \end{tabular}%
    }
  \label{main betae datasets}%
\end{table*}%

\subsection{Discussion on Computational Complexity}
\label{complexity}

In addition to the message-passing-based NSMP, all existing neural-symbolic methods adopt a step-by-step approach. Among these neural-symbolic methods, FIT \citep{yin2024rethinking} is the current SOTA method, which is a natural extension of QTO \citep{bai2023answering}. According to \citep{bai2023answering}, QTO outperforms previous neural-symbolic methods \citep{arakelyan2021complex, zhu2022neural} in both performance and efficiency. Consequently, FIT serves as the representative method for current step-by-step neural-symbolic approaches. 
In the following, we analyze the computational complexity of NSMP compared to FIT, highlighting why message-passing-based models enable more efficient inference than the step-by-step approach. Here, we focus on time complexity, while the analysis of space complexity is provided in App. \ref{space complexity}. 
\begin{proposition}
Disregarding sparsity, the time complexity of NSMP for any $\text{EFO}_1$ formula is approximately linear to $\mathcal{O}(|\mathcal{V}|^2)$. \label{prop:prop1}
\end{proposition}
The proof of the Proposition \ref{prop:prop1} is provided in App. \ref{proof1}. 
\begin{proposition}
For cyclic queries, NSMP offers more efficient inference than step-by-step FIT. \label{prop:prop2}
\end{proposition}
According to \citep{yin2024rethinking}, the complexity of FIT on cyclic queries is $\mathcal{O}(|\mathcal{V}|^n)$, where $n$ is the number of variables in the query graph. In contrast, the complexity of NSMP is approximately $\mathcal{O}(|\mathcal{V}|^2)$. This means that NSMP can provide more efficient inference on cyclic queries. The full proof is in App. \ref{proof2}. 
We also empirically verify Proposition \ref{prop:prop2}. As demonstrated in Sec. \ref{major}, NSMP achieves faster inference times than FIT on cyclic queries, with speedup ranging from 69$\times$ to over 150$\times$. 
\begin{proposition}
For acyclic queries, NSMP has better complexity and provides more efficient inference than FIT/QTO when considering sparsity. 
\label{prop:prop4}
\end{proposition}
Both FIT and QTO sparsify the neural adjacency matrix they use by setting appropriate thresholds. However, the relational adjacency matrix employed in NSMP only contains $0$s and $1$s. Therefore, in terms of the adjacency matrix, NSMP theoretically exhibits a more efficient complexity compared to FIT/QTO. This results in NSMP providing more efficient inference on acyclic queries than FIT/QTO. The full proof is in App. \ref{proof4}. We also empirically verify Proposition \ref{prop:prop4}. As demonstrated in Sec. \ref{major}, NSMP achieves at least a 10$\times$ speedup in inference time for acyclic queries in NELL995.

\section{Experiments}

\subsection{Experimental Settings}
\label{experimental setting}

\textbf{Datasets and Queries.} We evaluate our model on two popular KGs: FB15k-237 \citep{toutanova2015observed} and NELL995 \citep{xiong2017deeppath}. 
We follow \citep{chen2022fuzzy,xu2022neural} and exclude FB15k \citep{bordes2013translating} since it suffers from major test leakage \citep{toutanova2015observed, rossi2021knowledge}. 
For a fair comparison with previous works, we evaluate our model using both the datasets introduced by \citep{ren2020beta}, which we refer to as the \textbf{BetaE} datasets, and the datasets proposed by \citep{yin2024rethinking}, which we refer to as the \textbf{FIT} datasets. 
In particular, 
according to \citep{yin2024rethinking}, the ``pni'' query in the BetaE datasets is answered as the universal quantifier version. \citet{yin2024rethinking} maintains the ``pni'' query type but re-samples the answers according to their own definition. Therefore, we follow \citep{yin2024rethinking} and only evaluate the ``pni'' query with re-sampled answers, as defined in the FIT datasets. 
For the graph representation of the query types and related statistics, please refer to App. \ref{More Details about the Datasets}. 

\textbf{Evaluation Protocol.} 
We follow the evaluation scheme proposed in \citep{ren2020query2box,ren2020beta}. Specifically, the answers to complex queries are categorized into two types: easy and hard. Easy answers can be found by directly traversing the KG, whereas hard answers require non-trivial reasoning to be derived, i.e., non-trivial answers. Since our evaluation focuses on incomplete KGs, we concentrate on the model's ability to discover hard answers. 
Each hard answer of a query is ranked against non-answer entities to compute the Mean Reciprocal Rank (MRR).

\textbf{Baselines.} We consider the state-of-the-art CQA models from recent years as our baselines, including BetaE \citep{ren2020beta}, CQD \citep{arakelyan2021complex}, LogicE \citep{luus2021logic}, ConE \citep{zhang2021cone}, FuzzQE \citep{chen2022fuzzy}, GNN-QE \citep{zhu2022neural}, ENeSy \citep{xu2022neural}, $\text{CQD}^\mathcal{A}$ \citep{arakelyan2023adapting}, LMPNN \citep{wang2023logical}, QTO \citep{bai2023answering}, CLMPT \citep{zhang2024conditional}, and FIT \citep{yin2024rethinking}, where LMPNN and CLMPT are based on message passing. We also compare more neural CQA models on the BetaE datasets in App. \ref{Comparison with More Neural CQA Models on BetaE Datasets}. For details about the model, implementation, and experiments, please refer to App. \ref{Details about the Model, Implementation and Experiments}. 

\begin{table}[htbp]
  \centering
  \caption{MRR results on FIT datasets. Highlighted are the top \textcolor[rgb]{1, 0, 0}{\textbf{first}} and \textcolor[rgb]{0, 0, 1}{\textbf{second}} results.  }
  \resizebox{1\linewidth}{!}{
    \begin{tabular}{ccccccccccccc}
    \toprule
    \textbf{KG} & \textbf{Model} & \textbf{pni} & \textbf{2il} & \textbf{3il} & \textbf{2m} & \textbf{2nm} & \textbf{3mp} & \textbf{3pm} & \textbf{im} & \textbf{3c} & \textbf{3cm} & \textbf{AVG} \\
    \midrule
    \multirow{10}[4]{*}{FB15k-237} & BetaE & 9.0   & 25.0  & 40.1  & 8.6   & 6.7   & 8.6   & 6.8   & 12.3  & 25.2  & 22.9  & 16.5 \\
          & LogicE & 9.5   & 27.1  & 42.0  & 8.6   & 6.7   & 9.4   & 6.1   & 12.8  & 25.4  & 23.3  & 17.1 \\
          & ConE  & 10.8  & 27.6  & 43.9  & \textcolor[rgb]{0, 0, 1}{\textbf{9.6}} & 7.0   & 9.3   & 7.3   & 14.0  & 28.2  & 24.9  & 18.3 \\
          & QTO   & 12.1  & 28.9  & 47.9  & 8.5   & \textcolor[rgb]{0, 0, 1}{\textbf{10.7}} & \textcolor[rgb]{0, 0, 1}{\textbf{11.4}} & 6.5   & 17.9  & 38.3  & \textcolor[rgb]{0, 0, 1}{\textbf{35.4}} & 21.8 \\
          & CQD   & 7.7   & 29.6  & 46.1  & 6.0   & 1.7   & 6.8   & 3.3   & 12.3  & 25.9  & 23.8  & 16.3 \\
          & FIT   & \textcolor[rgb]{1, 0, 0}{\textbf{14.9}} & \textcolor[rgb]{1, 0, 0}{\textbf{34.2}} & \textcolor[rgb]{1, 0, 0}{\textbf{51.4}} & \textcolor[rgb]{1, 0, 0}{\textbf{9.9}} & \textcolor[rgb]{1, 0, 0}{\textbf{12.7}} & \textcolor[rgb]{1, 0, 0}{\textbf{11.9}} & \textcolor[rgb]{1, 0, 0}{\textbf{7.7}} & \textcolor[rgb]{1, 0, 0}{\textbf{19.6}} & \textcolor[rgb]{1, 0, 0}{\textbf{39.4}} & \textcolor[rgb]{1, 0, 0}{\textbf{37.3}} & \textcolor[rgb]{1, 0, 0}{\textbf{23.9}} \\
\cmidrule{2-13}          & \multicolumn{12}{l}{(Based on message passing)} \\
          & LMPNN & 10.7  & 28.7  & 42.1  & 9.4   & 4.2   & 9.8   & 7.2   & 15.4  & 25.3  & 22.2  & 17.5 \\
          & CLMPT & 10.1  & 31.0  & 48.5  & 8.7   & 7.8   & 10.1  & 6.1   & 15.8  & 30.2  & 28.5  & 19.7 \\
          & NSMP  & \textcolor[rgb]{0, 0, 1}{\textbf{13.4}} & \textcolor[rgb]{0, 0, 1}{\textbf{32.9}} & \textcolor[rgb]{0, 0, 1}{\textbf{51.2}} & 9.2   & 9.9   & \textcolor[rgb]{0, 0, 1}{\textbf{11.4}} & \textcolor[rgb]{0, 0, 1}{\textbf{7.5}} & \textcolor[rgb]{0, 0, 1}{\textbf{18.9}} & \textcolor[rgb]{0, 0, 1}{\textbf{39.0}} & 34.5  & \textcolor[rgb]{0, 0, 1}{\textbf{22.8}} \\
    \midrule
    \multirow{10}[4]{*}{NELL995} & BetaE & 7.5   & 43.3  & 64.6  & 29.0  & 5.3   & 8.7   & 14.4  & 29.5  & 36.1  & 33.7  & 27.2 \\
          & LogicE & 9.8   & 47.0  & 66.6  & 34.7  & 6.4   & 13.3  & 17.8  & 35.1  & 38.9  & 37.9  & 30.8 \\
          & ConE  & 10.3  & 42.1  & 65.8  & 32.4  & 7.0   & 12.6  & 16.8  & 34.4  & 40.2  & 38.2  & 30.0 \\
          & QTO   & 12.3  & 48.5  & 68.2  & \textcolor[rgb]{0, 0, 1}{\textbf{38.8}} & \textcolor[rgb]{0, 0, 1}{\textbf{12.3}} & \textcolor[rgb]{0, 0, 1}{\textbf{22.8}} & 19.3  & 41.1  & 45.4  & \textcolor[rgb]{0, 0, 1}{\textbf{43.9}} & 35.3 \\
          & CQD   & 7.9   & 48.7  & 68.0  & 31.7  & 1.5   & 12.9  & 13.8  & 33.9  & 38.8  & 35.9  & 29.3 \\
          & FIT   & \textcolor[rgb]{1, 0, 0}{\textbf{14.4}} & \textcolor[rgb]{1, 0, 0}{\textbf{53.3}} & \textcolor[rgb]{0, 0, 1}{\textbf{69.5}} & \textcolor[rgb]{1, 0, 0}{\textbf{42.1}} & \textcolor[rgb]{1, 0, 0}{\textbf{12.5}} & \textcolor[rgb]{1, 0, 0}{\textbf{24.0}} & \textcolor[rgb]{1, 0, 0}{\textbf{22.8}} & \textcolor[rgb]{0, 0, 1}{\textbf{41.5}} & \textcolor[rgb]{1, 0, 0}{\textbf{47.5}} & \textcolor[rgb]{1, 0, 0}{\textbf{45.3}} & \textcolor[rgb]{1, 0, 0}{\textbf{37.3}} \\
\cmidrule{2-13}          & \multicolumn{12}{l}{(Based on message passing)} \\
          & LMPNN & 11.6  & 43.9  & 62.3  & 35.6  & 6.2   & 15.9  & 19.3  & 38.3  & 39.1  & 34.4  & 30.7 \\
          & CLMPT & 12.5  & 48.7  & 68.2  & 36.6  & 7.5   & 19.0  & \textcolor[rgb]{0, 0, 1}{\textbf{19.9}} & 39.1  & 44.4  & 41.2  & 33.7 \\
          & NSMP  & \textcolor[rgb]{0, 0, 1}{\textbf{13.0}} & \textcolor[rgb]{0, 0, 1}{\textbf{52.4}} & \textcolor[rgb]{1, 0, 0}{\textbf{71.3}} & 37.6  & 11.5  & 21.7  & 18.3  & \textcolor[rgb]{1, 0, 0}{\textbf{41.7}} & \textcolor[rgb]{0, 0, 1}{\textbf{46.6}} & 42.4  & \textcolor[rgb]{0, 0, 1}{\textbf{35.7}} \\
    \bottomrule
    \end{tabular}%
    }
  \label{main fit datasets}%
\end{table}%

\subsection{Major Results}
\label{major}

Table \ref{main betae datasets} and Table \ref{main fit datasets} present the results of NSMP compared to neural and neural-symbolic CQA baselines on the BetaE and FIT datasets, respectively. 
It can be observed that NSMP outperforms other message-passing-based CQA models on both positive and negative queries, even without training on complex queries, and achieves a significant improvement on negative queries. 
For other baselines, NSMP outperforms most neural and neural-symbolic CQA models, achieving a strong performance. 
Specifically, although NSMP does not significantly outperform $\text{CQD}^\mathcal{A}$ on NELL995, NSMP has the advantage of not requiring training on complex queries, whereas $\text{CQD}^\mathcal{A}$ requires training on complex queries. A more detailed discussion validating the superiority of NSMP over $\text{CQD}^\mathcal{A}$ is provided in App. \ref{comparison with cqda}. 
Despite the achievement of second-best results compared to the state-of-the-art neural-symbolic model FIT, as discussed in Sec. \ref{complexity}, NSMP can offer superior inference efficiency. 
Specifically, we evaluate the relative speedup of NSMP over FIT in terms of inference time on FIT datasets to make a sharp contrast, as illustrated in Fig. \ref{figure3}. 
NSMP demonstrates faster inference times across all query types on both FB15k-237 and NELL995, with speedup ranging from 2$\times$ to over 150$\times$. 
Notably, for more complex cyclic queries, such as ``3c'' and ``3cm'', the relative speedup of NSMP becomes even more pronounced, as discussed in Sec. \ref{complexity}. 
Moreover, on the larger knowledge graph NELL995, NSMP exhibits even greater relative speedup, suggesting that NSMP offers better scalability compared to FIT. More experimental details on inference time can be found in App. \ref{Details about the Model, Implementation and Experiments}. In addition, We also provide a case study in App. \ref{case study}.

\begin{figure}
    \centering
    \includegraphics[width=0.8\linewidth]{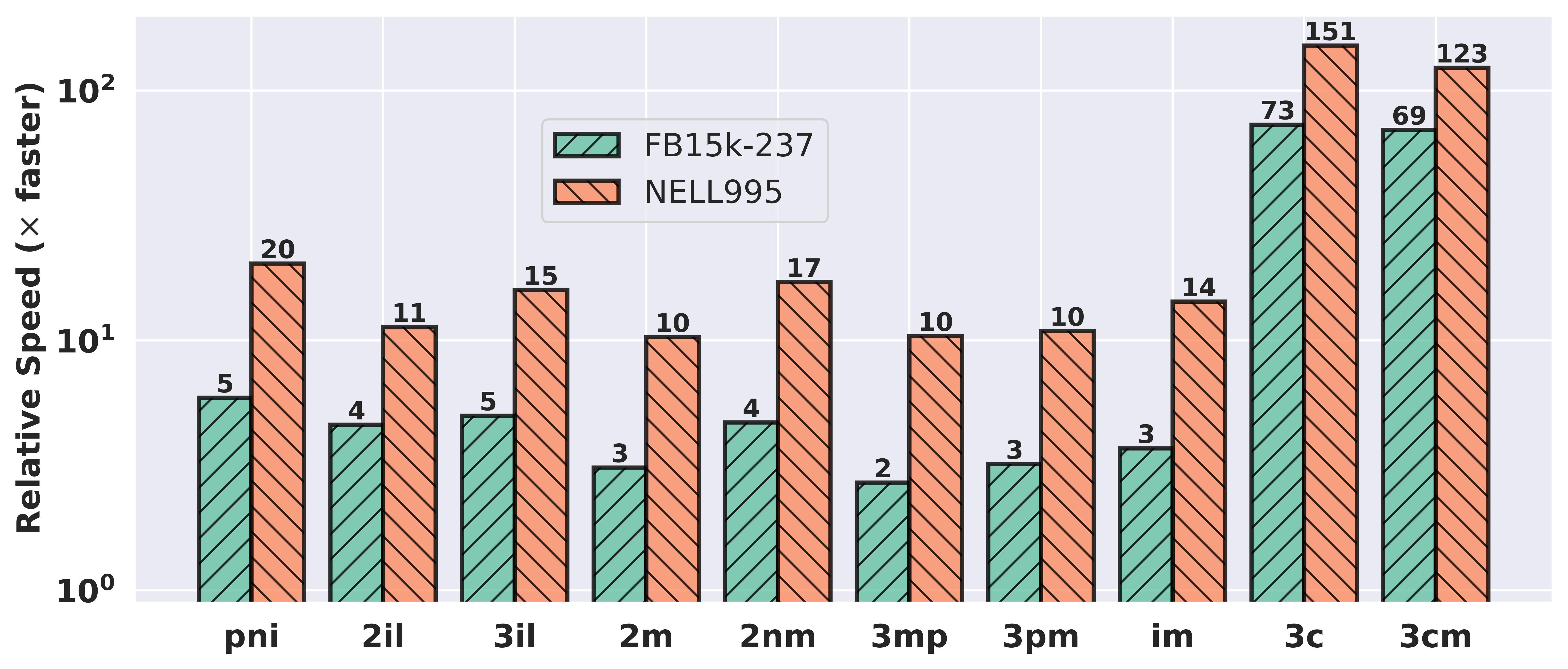}
    \caption{Relative speedup of NSMP over FIT in terms of inference time on FIT datasets. }
    \label{figure3}
\end{figure}

\subsection{Ablation Study}
\label{ablation}

We first explore the effect of the hyperparameter $L$ on performance. 
As noted in Sec. \ref{inference exp}, $L$ should be greater than or equal to $D$. 
Thus, we evaluate the performance for depths $D$, $D+1$, $D+2$, and $D+3$ on FB15k-237. 
As shown in Table \ref{depth}, NSMP achieves the best average performance when $L=D+1$. 
However, we observe that other depth choices can yield better results for specific query types. 
One approach is to manually select the most appropriate depth $L$ for different query types to achieve better average performance. As indicated in Table \ref{depth}, this manual choice achieves an average MRR result of $23.3$. 
But for simplicity, we set $L=D+1$ by default in our work.

\begin{table}[htbp]
    \centering
    \caption{MRR results of NSMP with different layers on FB15k-237. }
    \resizebox{0.9\linewidth}{!}{
    \begin{tabular}{cccccccccccc}
    \toprule
    \textbf{Model} & \textbf{pni} & \textbf{2il} & \textbf{3il} & \textbf{2m} & \textbf{2nm} & \textbf{3mp} & \textbf{3pm} & \textbf{im} & \textbf{3c} & \textbf{3cm} & \textbf{AVG} \\
    \midrule
   $ L=D$   & 11.9  & 29.1  & 49.2  & \textbf{9.2} & 9.9   & 11.4  & \textbf{7.5} & \textbf{18.9} & \textbf{40.0} & \textbf{37.9} & 22.5 \\
    $L=D+1$ & \textbf{13.4} & \textbf{32.9} & \textbf{51.2} & \textbf{9.2} & 9.9   & 11.4  & \textbf{7.5} & \textbf{18.9} & 39.1  & 34.5  & \textbf{22.8} \\
    $L=D+2$ & 12.7  & \textbf{32.9} & \textbf{51.2} & 7.9   & \textbf{10.5} & \textbf{11.5} & 5.6   & 16.8  & 38.5  & 33.5  & 22.1 \\
    $L=D+3$ & 13.3  & 31.8  & 50.2  & 7.9   & \textbf{10.5} & \textbf{11.5} & 5.6   & 16.8  & 38.1  & 31.8  & 21.8 \\
    \midrule
    Manual Choice & \textbf{13.4} & \textbf{32.9} & \textbf{51.2} & \textbf{9.2} & \textbf{10.5} & \textbf{11.5} & \textbf{7.5} & \textbf{18.9} & \textbf{40.0} & \textbf{37.9} & \textbf{23.3} \\
    \bottomrule
    \end{tabular}%
    }
    \label{depth}
\end{table}

To verify the effectiveness of the proposed dynamic pruning strategy, we conduct experiments on whether to perform dynamic pruning, and the results are shown in Table \ref{dp experiments}, where ``w/o DP'' indicates ``without using dynamic pruning'', AVG.(P), AVG.(N), and AVG.(F) represent the average scores for positive queries, negative queries on BetaE datasets, and the average scores on FIT datasets, respectively. 
It is found that the dynamic pruning strategy brings significant performance improvement, which shows the effectiveness of the strategy, indicating that the message from the variable node whose state is not updated is an unnecessary noise. 
In particular, it is observed that the dynamic pruning strategy achieves relatively higher performance gains on FIT datasets than on BetaE datasets. We believe this is because dynamic pruning is designed to filter out unnecessary noisy messages between variable nodes. Consequently, the number of variable nodes and their interactions in a query graph significantly influence the performance gains achievable through dynamic pruning. 
As illustrated in Fig. \ref{figure4}, \ref{figure5} of App. \ref{More Details about the Datasets}, the query types in FIT datasets involve more variable nodes compared to those in BetaE datasets. Moreover, the interactions between variable nodes are also more frequent in FIT queries, such as in ``3c" and ``3cm" queries. As a result, during message passing, query types in FIT datasets exhibit more noisy messages between variable nodes than in BetaE datasets. This explains why the dynamic pruning strategy achieves relatively higher performance gains on the FIT datasets. 
Furthermore, dynamic pruning further avoids unnecessary computation of noisy messages, thereby enhancing efficiency to a certain extent. Empirical results presented in App. \ref{dp efficiency} validate that dynamic pruning contributes to improved efficiency. 

For the discussion of other hyperparameters, please refer to App. \ref{Analysis on More Hyperparameters}. 

\begin{table}[htbp]
    \centering
    \caption{Average MRR results of NSMP with or without dynamic pruning. }
    \resizebox{0.6\linewidth}{!}{
    \begin{tabular}{ccccc}
    \toprule
    \textbf{KG} & \textbf{Model} & \textbf{AVG.(P)} & \textbf{AVG.(N)} & \textbf{AVG.(F)} \\
    \midrule
    \multirow{2}[2]{*}{FB15k-237} & NSMP w/o DP & 27.0  & 11.6  & 20.4 \\
          & NSMP  & \textbf{27.6} & \textbf{12.0} & \textbf{22.8} \\
    \midrule
    \multirow{2}[2]{*}{NELL995} & NSMP w/o DP & 32.1      & 12.1      & 31.2 \\
          & NSMP  & \textbf{32.4} & \textbf{12.4} & \textbf{35.7} \\
    \bottomrule
    \end{tabular}%
    }
    \label{dp experiments}
\end{table}

\section{Conclusion}

In this paper, we propose NSMP, a neural-symbolic message passing framework, to answer complex queries over KGs. 
By integrating neural and symbolic reasoning, NSMP can utilize fuzzy logic theory to answer arbitrary $\text{EFO}_1$ queries without the need for training on complex query datasets, while also offering interpretability through fuzzy sets. Moreover, we introduce a novel dynamic pruning strategy to filter out unnecessary noisy messages during message passing. In our ablation study, we validate the effectiveness of this strategy. Extensive experimental results demonstrate that NSMP outperforms other message passing CQA models, achieving a strong performance. 
Furthermore, we demonstrate the superiority of the message-passing-based NSMP in inference time over step-by-step neural-symbolic approaches through computational complexity analysis and empirical verification.

\section*{Impact Statement}
This paper aims to contribute to the advancement of knowledge graph reasoning, a rapidly evolving field with significant potential. While our research inevitably carries societal implications, we believe there are no specific issues that warrant particular attention in this discussion.

\balance


\bibliography{example_paper}
\bibliographystyle{icml2025}

\newpage
\appendix
\onecolumn

\section{Further Discussion for Related Work}
\label{further related work}

In the last decade and beyond, neural link predictors \citep{bordes2013translating, yang2015embedding,trouillon2016complex,dettmers2018convolutional,sun2018rotate} have been proposed to embed entities and relations into low-dimensional spaces to perform one-hop reasoning on incomplete KGs. 
Motivated by the success of neural link predictors and advancements in deep learning on sets \citep{zaheer2017deep}, \citet{hamilton2018embedding} proposed representing entity sets using low-dimensional embeddings, thereby answering complex queries in a step-by-step manner on incomplete KGs. This approach has since inspired a considerable body of work \citep{ren2020query2box,ren2020beta, zhang2021cone,choudhary2021probabilistic,bai2022query2particles,chen2022fuzzy,wang2023wasserstein}, with various methods utilizing different forms of vectors to represent entity sets. In addition, benefiting from the rapid progress of graph neural networks \citep{kipf2016semi, velivckovic2017graph} and transformers \citep{vaswani2017attention, ying2021transformers}, some studies \citep{daza2020message,liu2022mask,bai2023sequential} have represented complex queries as graphs or sequences, encoding the entire query at once rather than in a step-by-step manner. However, despite the effectiveness of the aforementioned neural models in handling complex queries on incomplete KGs by representing entity sets as neural embeddings, their performance on one-hop atomic queries remains limited.

To this end, \citet{wang2023logical} proposed a logical message passing model, which is most related to our work. 
This model leverages the pre-trained neural link predictor to infer intermediate embeddings for nodes in a query graph, interpreting these embeddings as logical messages. 
Through message passing paradigm \citep{xu2018how}, the free variable node embedding is updated to retrieve the answers. 
While effective on both one-hop atomic and multi-hop complex queries, logical message passing ignores the difference between constant and variable nodes, thus introducing noisy messages. To mitigate this, recent work \citep{zhang2024conditional} proposed a conditional message passing mechanism, which can be viewed as a pruning strategy regardless of the messages passed by variable nodes to constant nodes. 
However, this pruning strategy overlooks the noisy messages between variable nodes at the initial layers of message passing. 
In contrast, the dynamic pruning strategy employed in our model effectively eliminates these unnecessary noisy messages. 
Moreover, our proposed NSMP only needs to re-use a simple pre-trained neural link predictor without requiring training on complex query datasets, and offers interpretability through fuzzy sets.

In addition to the models mentioned above, there are several neural models enhanced with symbolic reasoning \citep{DBLP:conf/nips/SunAB0C20,arakelyan2021complex,zhu2022neural,xu2022neural,arakelyan2023adapting,bai2023answering} related to our work. 
These neural-symbolic models typically integrate neural link predictors with fuzzy logic to answer complex queries. 
Most of them depend on the operator tree \citep{wang2022benchmarking}, which, as noted in \citep{yin2024rethinking}, can only handle the existential first order logic formulas in an approximate way. 
In contrast, our proposed neural-symbolic model, which utilizes the query graph, enables a more natural and direct handling of these formulas.  
A recently proposed neural-symbolic model \citep{yin2024rethinking} introduces an algorithm that cuts nodes and edges step by step to handle the query graph. 
This model achieves SOTA performance but suffers from inefficiency. In contrast, our message-passing-based approach provides more efficient inference while achieving competitive performance, as discussed in Section \ref{complexity} and Section \ref{major}.

\section{Missing Proofs in the Main Paper}
\label{proof}

\subsection{Proof of Proposition \ref{prop:prop1}}
\label{proof1}

\begin{proof}
Since the computational bottleneck of NSMP lies in the symbolic-related parts, we focus on analyzing the time complexity of this aspect. 
According to Equation \ref{mu1}, \ref{mu2}, \ref{mu3}, \ref{mu4}, the complexity of symbolic one-hop inference is $\mathcal{O}(|\mathcal{V}|^2)$, so the complexity of neural-symbolic message encoding function $\varrho$ is approximately $\mathcal{O}(|\mathcal{V}|^2)$. This means that the complexity of message computation during message passing is approximately linear to $\mathcal{O}(|\mathcal{V}|^2)$. For the node state update process, symbolic state update and neural state update are involved, corresponding to Equation \ref{product conj}, \ref{s2z}, respectively. The complexity of the symbolic node state update is linear to $\mathcal{O}(|\mathcal{V}|)$, while the complexity of the neural node state update is $\mathcal{O}(|\mathcal{V}|d)$, where $d$ is the embedding dimension. 
Since we have $d\ll |\mathcal{V}|$, the total computational complexity of NSMP is approximately linear to $\mathcal{O}(|\mathcal{V}|^2)$. 
\end{proof}

\subsection{Proof of Proposition \ref{prop:prop2}}
\label{proof2}

\begin{proof}
According to \citep{yin2024rethinking}, FIT solves the acyclic query by continuously removing constants and leaf nodes, and the complexity of this process is approximately linear to $\mathcal{O}(|\mathcal{V}|^2)$. However, for cyclic queries, FIT needs to enumerate one variable within the cycle as a constant node, so the complexity is $\mathcal{O}(|\mathcal{V}|^n)$, where $n$ is the number of variables in the query graph. In contrast, the complexity of NSMP on cyclic queries is approximately linear to $\mathcal{O}(|\mathcal{V}|^2)$ (see Proposition \ref{prop:prop1}).
Cyclic queries are typically more complex and involve several variables, such as the ``3c'' and ``3cm'' query types proposed in \citep{yin2024rethinking}, which have three variables. In this case, the complexity of FIT is $\mathcal{O}(|\mathcal{V}|^3)>\mathcal{O}(|\mathcal{V}|^2)$. 
This means that NSMP can provide more efficient inference on cyclic queries. 
\end{proof}

\subsection{Proof of Proposition \ref{prop:prop4}}
\label{proof4}
We first analyze the time complexity of NSMP when considering sparsity. 
\begin{proposition}
Considering sparsity, the time complexity of NSMP is approximately $\mathcal{O}(K \cdot \frac{\mathcal{L}_N}{|\mathcal{V}|})<\mathcal{O}(|\mathcal{V}|^2)$, where $K$ is the number of non-zero elements in the sparse fuzzy vector and $\mathcal{L}_N$ is the total number of non-zero elements in the sparse relational adjacency matrix. 
\label{prop:prop3}
\end{proposition}
\begin{proof}
The symbolic one-hop inference function $\mu$ mainly involves sparse matrix multiplication between fuzzy vectors (or one-hot vectors) and the relational adjacency matrix. 
Due to the sparsity of the KG, the adjacency matrix, which contains only $0$s and $1$s, is extremely sparse. The complexity of the matrix can be represented as $\frac{\mathcal{L}_N}{|\mathcal{V}|}$, where $\mathcal{L}_N<|\mathcal{V}|^2$ is the total number of non-zero elements in the sparse relational adjacency matrix. In this case, $\frac{\mathcal{L}_N}{|\mathcal{V}|}$ can be interpreted as the average number of non-zero elements per row in this matrix.
In addition, since we adopt a thresholded normalization function (see Equation \ref{normalized function}), the fuzzy vector is also sparse. Therefore, 
the time complexity of $\mu$ is $\mathcal{O}(K \cdot \frac{\mathcal{L}_N}{|\mathcal{V}|})$, where $K$ is the number of non-zero elements in the sparse fuzzy vector. 
Since the computation of $\mu$ is the computational bottleneck of NSMP, the time complexity of NSMP is approximately $\mathcal{O}(K \cdot \frac{\mathcal{L}_N}{|\mathcal{V}|})$ while considering sparsity. 
\end{proof}
Next, we provide the proof of Proposition \ref{prop:prop4} based on Proposition \ref{prop:prop3}. 
\begin{proof}
FIT is a natural extension of QTO, and the time complexity of these two methods is equivalent for acyclic queries that include existential variables \citep{yin2024rethinking}. Therefore, we focus on analyzing the time complexity of NSMP and QTO on these acyclic queries while considering sparsity. According to \citep{bai2023answering}, due to the sparsity of the KG, the time complexity of QTO is actually $\mathcal{O}(|T^*(v_k)>0| \cdot \frac{\mathcal{L}_Q}{|\mathcal{V}|})$, where $T^*(v_k)$ is a sparse vector, $\mathcal{L}_Q$ is the total number of non-zero elements in the sparse neural adjacency matrix of QTO. 
For NSMP, the time complexity is approximately $\mathcal{O}(K \cdot \frac{\mathcal{L}_N}{|\mathcal{V}|})$ (see Proposition \ref{prop:prop3}). Since the sparsity of the fuzzy vector in NSMP and the vector $T^*(v_k)$ in QTO are both determined by the chosen threshold, we can assume that with appropriately set hyperparameters, $K = |T^*(v_k)>0|$. For $\mathcal{L}_Q$, the neural adjacency matrix in QTO can be sparsified by setting an appropriate threshold such that $\mathcal{L}_Q < |\mathcal{V}|^2$. However, the relational adjacency matrix used in NSMP only contains $0$s and $1$s. In theory, the neural adjacency matrix in QTO only achieves the same sparsity as the relational adjacency matrix used in NSMP when the threshold is set to $1$. However, setting the threshold to $1$ in QTO would make no sense, as it would strip the model of its neural reasoning capability. This suggests that, in terms of the adjacency matrix, NSMP theoretically has a more efficient computational complexity compared to QTO, with $\frac{\mathcal{L}_N}{|\mathcal{V}|} < \frac{\mathcal{L}_Q}{|\mathcal{V}|} < |\mathcal{V}|$. Thus, when considering sparsity, NSMP has a better computational complexity for acyclic queries compared to FIT and QTO.
\end{proof}

\section{Analysis on Space Complexity}
\label{space complexity}
Similar to the time complexity analysis in Section \ref{complexity}, for the space complexity analysis, we still focus on the comparison with FIT \citep{yin2024rethinking}. 
\begin{proposition}
Considering sparsity, NSMP achieves an approximate space complexity of $\mathcal{O}((|\mathcal{V}|+|\mathcal{R}|)d+|\mathcal{R}|\cdot \mathcal{L}_N)$,
with $d$ representing the embedding dimension and achieving superior space complexity compared to FIT. 
\label{prop:prop5}
\end{proposition}
\begin{proof}
For the space complexity of neural components, both NSMP and FIT use a pre-trained neural link predictor with a complexity of $\mathcal{O}((|\mathcal{V}|+|\mathcal{R}|)d)$, where $d$ is the embedding dimension.
The symbolic one-hop inference component of NSMP utilizes relational adjacency matrices, which contain $|\mathcal{R}|\cdot |\mathcal{V}|^2$ entries. 
However, due to the sparsity of KG, most entries are $0$. With the help of sparse matrix techniques, the adjacency matrices can be stored efficiently. 
The neural adjacency matrix used in FIT also contains $|\mathcal{R}|\cdot |\mathcal{V}|^2$ entries and can be efficiently stored by setting appropriate thresholds. 
However, according to the proof of Proposition \ref{prop:prop4}, the relational adjacency matrix used in NSMP exhibits higher sparsity than the neural adjacency matrix used in FIT, i.e., $\mathcal{L}_N<\mathcal{L}_Q=\mathcal{L}_F<|\mathcal{V}|^2$, where $\mathcal{L}_F$ is the total number of non-zero elements in the sparse neural adjacency matrix of FIT. Consequently, NSMP achieves superior space complexity compared to FIT, with its space complexity approximately given by $\mathcal{O}((|\mathcal{V}|+|\mathcal{R}|)d+|\mathcal{R}|\cdot \mathcal{L}_N)$. 
\end{proof}

\section{More Details about the Datasets}
\label{More Details about the Datasets}

The statistics of the two knowledge graphs used in our experiment are shown in Table \ref{KG stat}. 
As we described in Section \ref{experimental setting}, we evaluate our model using both BetaE \citep{ren2020beta} datasets and FIT \citep{yin2024rethinking} datasets. 
The statistics for the BetaE datasets are shown in Table \ref{BetaE stat}.  
Since our model does not require training on complex queries, we use only the test split of the BetaE datasets. 
In addition, as stated in Section \ref{experimental setting}, the ``pni'' query in the BetaE datasets is answered as the universal quantifier version. 
Consequently, we do not evaluate ``pni'' queries in BetaE datasets, but in FIT datasets. 
Specifically, the query types we evaluated in BetaE datasets are shown in Figure \ref{figure4}. 
For the FIT datasets, which contain only test queries, the statistics are shown in Table \ref{FIT stat}. The query types of the FIT datasets are shown in Figure \ref{figure5}. 

\begin{table}[htbp]
  \centering
  \caption{Statistics of knowledge graphs as well as training, validation and test edge splits. }
  \resizebox{0.9\linewidth}{!}{
    \begin{tabular}{ccccccc}
    \toprule
    \textbf{Knowledge Graph} & \textbf{Entities} & \textbf{Relations} & \textbf{Training Edges} & \textbf{Val Edges} & \textbf{Test Edges} & \textbf{Total Edges} \\
    \midrule
    FB15k-237 & 14,505 & 237   & 272,115 & 17,526 & 20,438 & 310,079 \\
    \midrule
    NELL995 & 63,361 & 200   & 114,213 & 14,324 & 14,267 & 142,804 \\
    \bottomrule
    \end{tabular}%
    }
  \label{KG stat}%
\end{table}%
\begin{table}[htbp]
  \centering
  \caption{Statistics of different query types used in the BetaE datasets.}
  \resizebox{0.85\linewidth}{!}{
    \begin{tabular}{ccccccc}
    \toprule
    \multirow{2}{*}{\textbf{Knowledge Graph}} & \multicolumn{2}{c}{\textbf{Training Queries}} & \multicolumn{2}{c}{\textbf{Validation Queries}} & \multicolumn{2}{c}{\textbf{Test Queries}} \\
    \cmidrule(lr){2-3} \cmidrule(lr){4-5} \cmidrule(lr){6-7}
     & 1p/2p/3p/2i/3i & 2in/3in/inp/pin/pni & 1p    & Others & 1p    & Others \\
    \midrule
    FB15k-237 & 149,689 & 14,968 & 20,101 & 5,000 & 22,812 & 5,000 \\
    NELL995 & 107,982 & 10,798 & 16,927 & 4,000 & 17,034 & 4,000 \\
    \bottomrule
    \end{tabular}%
    }
  \label{BetaE stat}%
\end{table}%
\begin{table}[htbp]
  \centering
  \caption{Statistics of different query types used in the FIT datasets.}
  \resizebox{0.9\linewidth}{!}{
    \begin{tabular}{ccccccccccc}
    \toprule
    \textbf{Knowledge Graph} & \textbf{pni} & \textbf{2il} & \textbf{3il} & \textbf{2m} & \textbf{2nm} & \textbf{3mp} & \textbf{3pm} & \textbf{im} & \textbf{3c} & \textbf{3cm} \\
    \midrule
    FB15k-237 & 5,000 & 5,000 & 5,000 & 5,000 & 5,000 & 5,000 & 5,000 & 5,000 & 5,000 & 5,000 \\
    NELL995 & 4,000 & 5,000 & 5,000 & 5,000 & 5,000 & 5,000 & 5,000 & 5,000 & 5,000 & 5,000 \\
    \bottomrule
    \end{tabular}%
    }
  \label{FIT stat}%
\end{table}%

\begin{figure}
    \centering
    \includegraphics[width=0.9\textwidth]{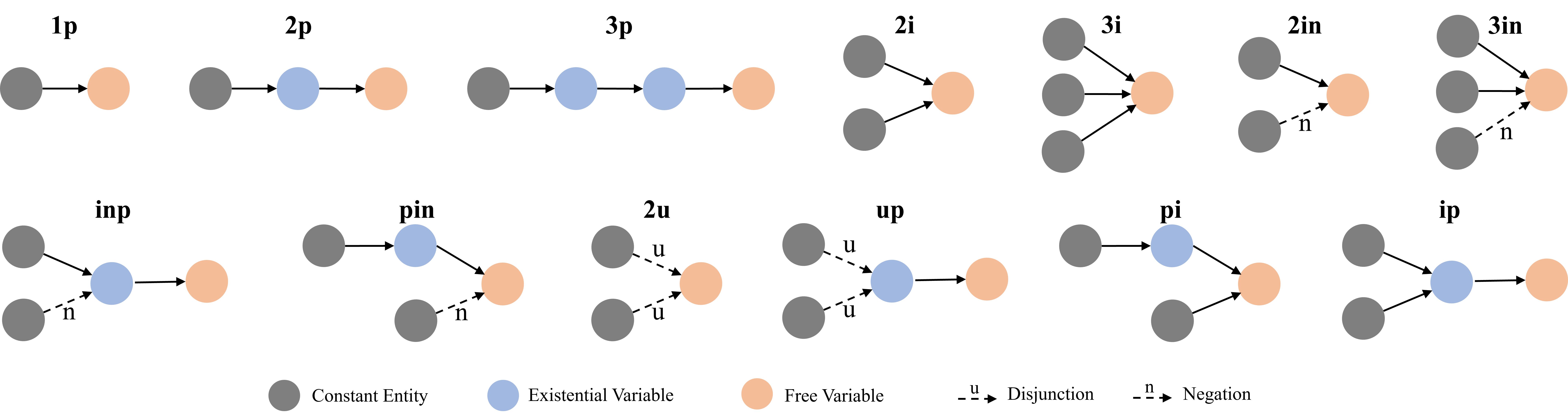}
    \caption{Graphical representation of the query types of the BetaE dataset considered in our experiment, where $p$, $i$, $u$, and $n$ represent projection, intersection, union, and negation, respectively. }
    \label{figure4}
\end{figure}

\begin{figure}
    \centering
    \includegraphics[width=0.6\textwidth]{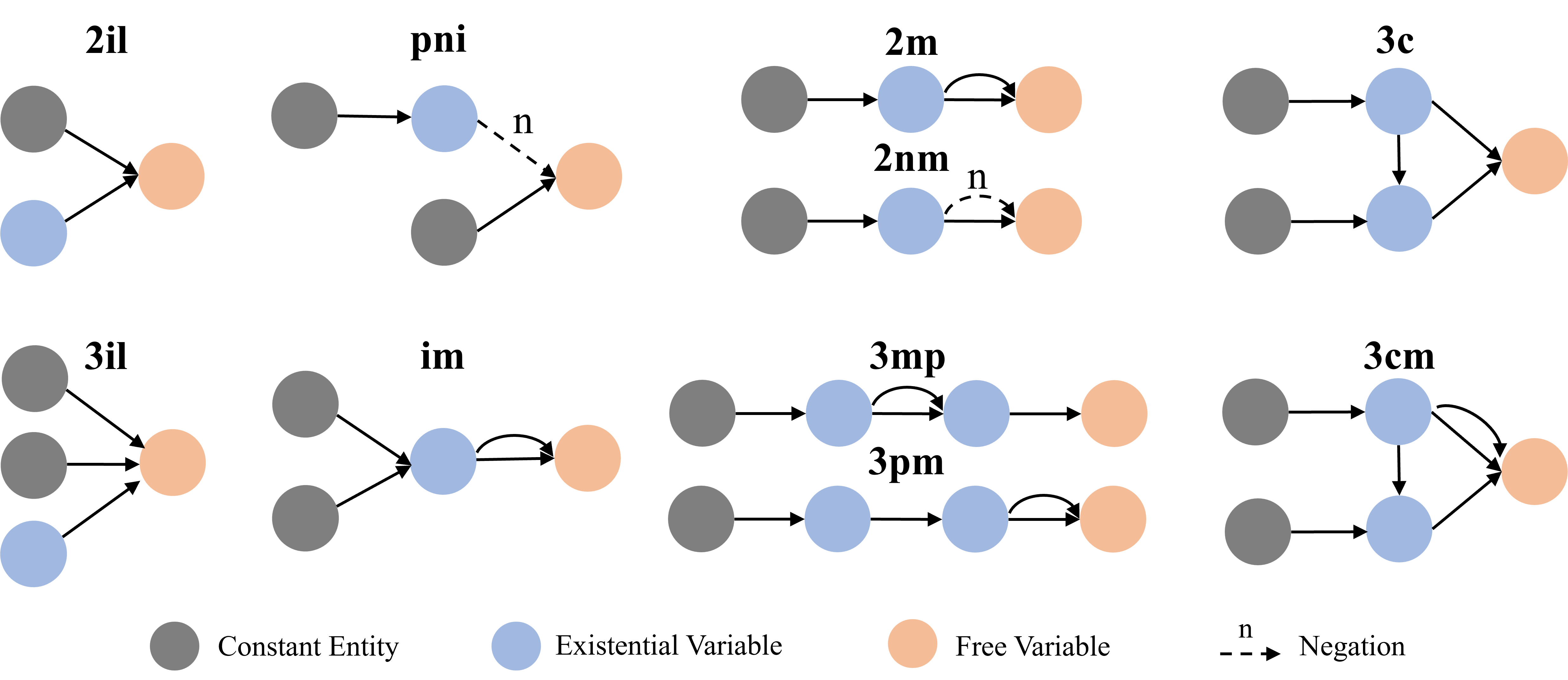}
    \caption{Graphical representation of the query types of the FIT dataset considered in our experiment, where $l$, $m$, and $c$ represent existential leaf, multi graph, and circle, respectively. }
    \label{figure5}
\end{figure}

\section{Details about the Model, Implementation and Experiments}
\label{Details about the Model, Implementation and Experiments}

We use NVIDIA RTX 3090 (24GB) and NVIDIA A100 (40GB) to conduct all of our experiments. Following \citep{arakelyan2021complex, wang2023logical, yin2024rethinking}, we select ComplEx-N3 \citep{trouillon2016complex, lacroix2018canonical} as the neural link predictor and use the checkpoints released by \citep{arakelyan2021complex} for a fair comparison. The rank of ComplEx-N3 is 1,000, and the epoch for the checkpoints is 100. To determine the optimal hyperparameters of NSMP, we employ grid search. Specifically, the value for $\alpha$ is selected from \{1, 10, 100, 1000, 10000, $|\mathcal{V}|$\}, 
for $\epsilon$ from \{1e-2, 1e-4, 1e-6, 1e-8, 1e-10, 1e-12, 1e-14, 1e-16\}, 
for $\lambda$ from \{0, 0.1, 0.2, 0.3, 0.4, 0.5, 0.6, 0.7, 0.8, 0.9, 1.0\}, and for $L$ from \{$D$, $D+1$, $D+2$, $D+3$\}. After experimental analyses on hyperparameters, we choose $\lambda$ as $0.3$ for FB15k-237 and $0.1$ for NELL995. 
The threshold $\epsilon$ is 1e-14, the number of layers $L$ is $D+1$, and $\alpha$ is $100$ for FB15k-237 and $1000$ for NELL995.

In the case of choosing ComplEx-N3 as the neural link predictor, according to \citep{wang2023logical}, given the corresponding complex embeddings, the neural message encoding function $\rho$ can be derived into the following closed-form expressions for all four cases: 
\begin{align}
    \rho (t,r,D_{t\to h},Pos) &= \frac{\overline{r}\otimes t}{\sqrt{3\beta \left \| \overline{r}\otimes t \right \|  } },  \label{rho11} \\
    \rho (h,r,D_{h\to t},Pos) &= \frac{r\otimes h}{\sqrt{3\beta \left \| r\otimes h \right \|  } },  \label{rho21} \\
    \rho (t,r,D_{t\to h},Neg) &= \frac{-\overline{r}\otimes t}{\sqrt{3\beta \left \| \overline{r}\otimes t \right \|  } },  \label{rho31} \\
    \rho (h,r,D_{h\to t},Neg) &= \frac{-r \otimes h}{\sqrt{3\beta \left \| r\otimes h \right \|  } },  \label{rho41} 
\end{align}
where $\beta$ is a hyperparameter that needs to be determined. In our application, we follow previous works \citep{wang2023logical,zhang2024conditional} and let $3\beta \left \| \cdot \right \| = 1$ for simplicity. Thus, $\rho$ is simplified to the following expressions for four cases: 
\begin{align}
    \rho (t,r,D_{t\to h},Pos) &:= \overline{r}\otimes t,  \label{app rho1} \\
    \rho (h,r,D_{h\to t},Pos) &:= r \otimes h,  \label{app rho2} \\ 
    \rho (t,r,D_{t\to h},Neg) &:= -\overline{r}\otimes t,  \label{app rho3} \\
    \rho (h,r,D_{h\to t},Neg) &:= -r \otimes h.  \label{app rho4}
\end{align}
For closed-form expressions of other neural link predictors, please refer to \citep{wang2023logical}.

For the experiments evaluating the relative speedup of NSMP over FIT \citep{yin2024rethinking} in terms of inference time on the FIT datasets, we conduct the experiments on an NVIDIA A100 GPU. Specifically, we measure the average time required by NSMP and FIT to process each type of test query on the FIT datasets. To ensure a fair comparison, the batch size during the testing phase is set to $1$. The relative speedup is presented in Figure \ref{figure3}, and detailed inference times are provided in Table \ref{infer time stat}.

\begin{table}[htbp]
  \centering
  \caption{Inference time (ms/query) on each query type on FIT datasets, evaluated on one NVIDIA A100 GPU. }
  \resizebox{0.85\linewidth}{!}{
    \begin{tabular}{cccccccccccc}
    \toprule
    \textbf{Knowledge Graph} & \textbf{Model} & \textbf{pni} & \textbf{2il} & \textbf{3il} & \textbf{2m} & \textbf{2nm} & \textbf{3mp} & \textbf{3pm} & \textbf{im} & \textbf{3c} & \textbf{3cm} \\
    \midrule
    \multirow{2}[2]{*}{FB15k-237} & FIT   & 93.4  & 48.4  & 73.4  & 47.8  & 71.2  & 69.4  & 74.4  & 69.8  & 2482.2 & 2824.2 \\
          & NSMP  & 15.8  & 10.6  & 14.6  & 15.2  & 15.2  & 25.6  & 23.4  & 18.8  & 34.0  & 40.6 \\
    \midrule
    \multirow{2}[2]{*}{NELL995} & FIT   & 843.0 & 419.0 & 630.6 & 422.6 & 649.6 & 636.0 & 632.4 & 634.2 & 11832.4 & 11149.9 \\
          & NSMP  & 41.5  & 37.0  & 39.6  & 41.0  & 38.0  & 61.2  & 57.8  & 44.2  & 78.2  & 90.4 \\
    \bottomrule
    \end{tabular}%
    }
  \label{infer time stat}%
\end{table}%

\section{Comparison with More Neural CQA Models on BetaE Datasets}
\label{Comparison with More Neural CQA Models on BetaE Datasets}

To further evaluate the performance of NSMP, 
we also consider comparing more neural CQA models on BetaE \citep{ren2020beta} datasets, including Q2P \citep{bai2022query2particles}, MLP \citep{amayuelas2022neural}, GammaE \citep{yang2022gammae}, CylE \citep{nguyen2023cyle}, WRFE \citep{wang2023wasserstein}, and Pathformer \citep{zhang2024pathformer}. The reported MRR results are from these papers \citep{wang2023wasserstein,nguyen2023cyle,zhang2024pathformer}. As shown in the results in Table \ref{more CQA baselines}, our model reaches the best performance across all query types on both FB15k-237 and NELL995, indicating the effectiveness of NSMP.

\begin{table}[htbp]
  \centering
  \caption{MRR results of other neural CQA models and our model on BetaE datasets. The average score is calculated separately among positive and negative queries. Highlighted are the top \textbf{first} results. }
  \resizebox{0.9\linewidth}{!}{
    \begin{tabular}{ccccccccccccccccc}
    \toprule
    \textbf{KG} & \textbf{Model} & \textbf{1p} & \textbf{2p} & \textbf{3p} & \textbf{2i} & \textbf{3i} & \textbf{pi} & \textbf{ip} & \textbf{2u} & \textbf{up} & \textbf{AVG.(P)} & \textbf{2in} & \textbf{3in} & \textbf{inp} & \textbf{pin} & \textbf{AVG.(N)} \\
    \midrule
    \multirow{7}[2]{*}{FB15k-237} & Q2P   & 39.1  & 11.4  & 10.1  & 32.3  & 47.7  & 24.0    & 14.3  & 8.7   & 9.1   & 21.9  & 4.4   & 9.7   & 7.5   & 4.6   & 6.6 \\
          & MLP   & 42.7  & 12.4  & 10.6  & 31.7  & 43.9  & 24.2  & 14.9  & 13.7  & 9.7   & 22.6  & 6.6   & 10.7  & 8.1   & 4.7   & 7.5 \\
          & GammaE & 43.2  & 13.2  & 11.0    & 33.5  & 47.9  & 27.2  & 15.9  & 13.9  & 10.3  & 24.0  & 6.7   & 9.4   & 8.6   & 4.8   & 7.4 \\
          & Pathformer & 44.8  & 12.9  & 10.6  & 34.2  & 47.3  & 26.2  & 17.0  & 14.9  & 10.0  & 24.2  & 6.4   & 11.6  & 8.3   & 4.7   & 7.8 \\
          & CylE  & 42.9  & 13.3  & 11.3  & 35.0  & 49.0  & 27.0  & 15.7  & 15.3  & 11.2  & 24.5  & 4.9   & 8.3   & 8.2   & 3.7   & 6.3 \\
          & WRFE  & 44.1  & 13.4  & 11.1  & 35.1  & 50.1  & 27.4  & 17.2  & 13.9  & 10.9  & 24.8  & 6.9   & 11.2  & 8.5   & 5.0   & 7.9 \\
          & NSMP  & \textbf{46.7} & \textbf{15.1} & \textbf{12.3} & \textbf{38.7} & \textbf{52.2} & \textbf{31.2} & \textbf{23.3} & \textbf{17.2} & \textbf{11.9} & \textbf{27.6} & \textbf{11.9} & \textbf{17.6} & \textbf{10.8} & \textbf{7.9} & \textbf{12.0} \\
    \midrule
    \multirow{7}[2]{*}{NELL995} & Q2P   & 56.5  & 15.2  & 12.5  & 35.8  & 48.7  & 22.6  & 16.1  & 11.1  & 10.4  & 25.5  & 5.1   & 7.4   & 10.2  & 3.3   & 6.5 \\
          & MLP   & 55.2  & 16.8  & 14.9  & 36.4  & 48.0  & 22.7  & 18.2  & 14.7  & 11.3  & 26.5  & 5.1   & 8.0     & 10.0    & 3.6   & 6.7 \\
          & GammaE & 55.1  & 17.3  & 14.2  & 41.9  & 51.1  & 26.9  & 18.3  & 15.1  & 11.2  & 27.9  & 6.3   & 8.7   & 11.4  & 4.0     & 7.6 \\
          & Pathformer & 56.4  & 17.4  & 14.9  & 39.9  & 50.4  & 26.0  & 19.4  & 14.4  & 11.1  & 27.8  & 5.1   & 8.6   & 10.3  & 3.9   & 7.0 \\
          & CylE  & 56.5  & 17.5  & 15.6  & 41.4  & 51.2  & 27.2  & 19.6  & 15.7  & 12.3  & 28.6  & 5.6   & 7.5   & 11.2  & 3.4   & 6.9 \\
          & WRFE  & 58.6  & 18.6  & 16.0    & 41.2  & 52.7  & 28.4  & 20.7  & 16.1  & 13.2  & 29.5  & 6.9   & 8.8   & 12.5  & 4.1   & 8.1 \\
          & NSMP  & \textbf{60.8} & \textbf{21.6} & \textbf{17.6} & \textbf{44.2} & \textbf{53.6} & \textbf{33.7} & \textbf{26.7} & \textbf{19.1} & \textbf{14.4} & \textbf{32.4} & \textbf{12.4} & \textbf{15.5} & \textbf{13.7} & \textbf{7.9} & \textbf{12.4} \\
    \bottomrule
    \end{tabular}%
    }
  \label{more CQA baselines}%
\end{table}%

\section{Analysis on More Hyperparameters}
\label{Analysis on More Hyperparameters}

For the evaluation of hyperparameters $\epsilon$, $\alpha$, and $\lambda$, we conduct experiments on FB15k-237. Specifically, we compare the effects of different hyperparameters on model performance under default hyperparameter settings. 
We evaluate the performance of various settings for the hyperparameter $\epsilon$ on the BetaE and FIT datasets. Specifically, we report the average MRR results for different $\epsilon$ values on FB15k-237, as presented in Table \ref{epsilon experi}, where AVG.(P), AVG.(N), and AVG.(F) represent the average scores for positive queries, negative queries on the BetaE datasets, and the average scores on the FIT datasets, respectively. 
It can be seen from these MRR results that $\epsilon$ of the threshold normalization function (see Equation \ref{normalized function}) has a significant impact on the model performance. 
In our main experiments, we set $\epsilon=1e-14$ by default, as this setting achieved the best performance in the hyperparameter experiments. 
For the hyperparameter $\alpha$, since it only influences the results of the negative queries, we evaluate $\alpha$ using the negative queries from both the BetaE and FIT datasets. As shown in the results in Table \ref{alpha experi}, $\alpha$ has a minor impact on the model performance, but $\alpha=100$ achieves slightly better results, so we set $\alpha=100$ for the experiments on FB15k-237. 
Similarly, we evaluate the performance of various settings for the hyperparameter $\lambda$ on the BetaE and FIT datasets and report the average MRR results, as shown in Table \ref{lambda experi}. Notably, except for $\lambda=0$, the other hyperparameter settings exhibit comparable performance. 
These results suggest that relying solely on aggregated neural representations (i.e., the $\lambda=0$ case) is insufficient for addressing complex queries. We adopt a default setting of  $\lambda=0.3$, as it achieves a relatively balanced performance. 

In particular, even when applying the hyperparameter configuration optimized for FB15k-237 (i.e., $\lambda=0.3$, $\alpha=100$) to NELL995, NSMP maintains stable performance on NELL995. The results are shown in Table \ref{robust}. These results offer a robust default hyperparameter configuration.

\begin{table}[htbp]
  \centering
  \caption{Average MRR results for different hyperparameters $\epsilon$ on FB15k-237.}
  \resizebox{0.4\linewidth}{!}{
    \begin{tabular}{cccc}
    \toprule
    \textbf{Model} & \textbf{AVG.(P)} & \textbf{AVG.(N)} & \textbf{AVG.(F)} \\
    \midrule
    $\epsilon=1e-2$   & 10.5  & 0.1   & 6.0 \\
    $\epsilon=1e-4$ & 18.4  & 0.1  & 12.6 \\
    $\epsilon=1e-6$ & 24.5  & 2.2  & 17.3 \\
    $\epsilon=1e-8$ & 26.9  & 8.2  & 20.7 \\
    $\epsilon=1e-10$ & 27.6  & 11.0  & 22.2 \\
    $\epsilon=1e-12$ & 27.6  & 12.0  & 22.7 \\
    $\epsilon=1e-14$ & 27.6  & 12.0  & 22.8 \\
    $\epsilon=1e-16$ & 27.6 & 12.0  & 22.8 \\
    \bottomrule
    \end{tabular}%
    }
  \label{epsilon experi}%
\end{table}%

\begin{table}[htbp]
  \centering
  \caption{MRR results for different hyperparameters $\alpha$ on FB15k-237.}
  \resizebox{0.5\linewidth}{!}{
    \begin{tabular}{cccccccc}
    \toprule
    \textbf{Model} & \textbf{2in} & \textbf{3in} & \textbf{inp} & \textbf{pin} & \textbf{pni} & \textbf{2nm} & \textbf{AVG} \\
    \midrule
    $\alpha=1$   & 11.9  & 17.6  & 10.7  & 8.0   & 13.0  & 8.9   & 11.7 \\
    $\alpha=10$  & 11.9  & 17.6  & 10.7  & 8.0   & 12.8  & 9.5   & 11.8 \\
    $\alpha=100$ & 11.9  & 17.6  & 10.8  & 7.9   & 13.4  & 9.9   & 11.9 \\
    $\alpha=1000$ & 11.9  & 17.6  & 10.8  & 7.9   & 13.7  & 9.6   & 11.9 \\
    $\alpha=10000$ & 11.7  & 17.6  & 10.8  & 7.8   & 13.9  & 9.0   & 11.8 \\
    $\alpha=|\mathcal{V}|$ & 11.7  & 17.5  & 10.8  & 7.8   & 13.7  & 8.4   & 11.7 \\
    \bottomrule
    \end{tabular}%
    }
  \label{alpha experi}%
\end{table}%

\begin{table}[htbp]
  \centering
  \caption{Average MRR results for different hyperparameters $\lambda$ on FB15k-237.}
  \resizebox{0.4\linewidth}{!}{
    \begin{tabular}{cccc}
    \toprule
    \textbf{Model} & \textbf{AVG.(P)} & \textbf{AVG.(N)} & \textbf{AVG.(F)} \\
    \midrule
    $\lambda=0$   & 19.8  & 5.5   & 17.4 \\
    $\lambda=0.1$ & 27.5  & 12.0  & 22.6 \\
    $\lambda=0.2$ & 27.6  & 12.0  & 22.7 \\
    $\lambda=0.3$ & 27.6  & 12.0  & 22.8 \\
    $\lambda=0.4$ & 27.6  & 12.0  & 22.8 \\
    $\lambda=0.5$ & 27.5  & 12.0  & 22.8 \\
    $\lambda=0.6$ & 27.5  & 12.0  & 22.8 \\
    $\lambda=0.7$ & 27.5  & 12.0  & 22.8 \\
    $\lambda=0.8$ & 27.4  & 12.0  & 22.8 \\
    $\lambda=0.9$ & 27.4  & 12.0  & 22.8 \\
    $\lambda=1.0$ & 27.4  & 12.0  & 22.9 \\
    \bottomrule
    \end{tabular}%
    }
  \label{lambda experi}%
\end{table}%

\begin{table}[htbp]
  \centering
  \caption{Average MRR results of NSMP with different hyperparameter configurations on NELL995. }
    \begin{tabular}{cccc}
    \toprule
    \textbf{Model} & \textbf{AVG.(P)} & \textbf{AVG.(N)} & \textbf{AVG.(F)} \\
    \midrule
    $\lambda=0.3$, $\alpha=100$  & 32.3  & 12.4  & 35.6 \\
    $\lambda=0.1$, $\alpha=1000$ (Best Choice) & 32.4  & 12.4  & 35.7 \\
    \bottomrule
    \end{tabular}%
  \label{robust}%
\end{table}%

\section{Case Study}
\label{case study}

To verify whether NSMP can provide interpretability, we sample an ``ip'' query from FB15k-237 to visualize the corresponding entity ranking derived from the final fuzzy state of each variable after message passing. Specifically, the $\text{EFO}_1$ formula for the ``ip'' query we sampled is $\exists x, r_1(x, e_1)\wedge r_2(x, e_2)\wedge r_3(x,y)$, where $e_1$ is \textit{Adventure Film}, $e_2$ is \textit{The Expendables}, $r_1$ is \textit{Genre}, $r_2$ is \textit{Prequel}, and $r_3$ is \textit{Film Regional Debut Venue}. 
This query has the following hard answers: \textit{Paris}, \textit{Buenos Aires}, \textit{Madrid}, \textit{Los Angeles}, \textit{London} and \textit{Belgrade}. 
After applying NSMP to the query, for each variable node, we find the top five entities with the highest probability from its final state based on Equation \ref{ns_eva} and a softmax operation. The results are shown in Table \ref{case study table}. The results indicate that the fuzzy set of each variable can be used to represent its membership degrees across all entities, thereby providing interpretability. 
Although NSMP cannot sample a specific reasoning path like step-by-step methods \cite{arakelyan2023adapting, bai2023answering}, such as the path sampled by beam search \citep{arakelyan2021complex}, the introduction of fuzzy logic allows each variable's state to be represented as a fuzzy vector that defines a fuzzy set. These fuzzy vectors can be leveraged to enhance interpretability. Compared to neural embeddings as representations of variable states, fuzzy vectors offer a more intuitive explanation of the current state of variables, thereby improving interpretability.

\begin{table}[htbp]
  \centering
  \caption{The top five entities with the highest probability for each variable and their corresponding probabilities. $\checkmark$ indicates that the correct entity is hit. }
    \begin{tabular}{cc|cc}
    \toprule
    \multicolumn{2}{c|}{$x$} & \multicolumn{2}{c}{$y$} \\
    \midrule
    Entity & Probability & Entity & Probability \\
    \midrule
    \textit{The Expendables 2} $\checkmark$ & $1.0$     & \textit{Buenos Aires} $\checkmark$ & $0.36$ \\
    \textit{Green Lantern} & $almost\ 0$ & \textit{Los Angeles} $\checkmark$ & $0.32$ \\
    \textit{The Dark Knight} & $almost\ 0$ & \textit{London} $\checkmark$ & $0.29$ \\
    \textit{Battle Royale} & $almost\ 0$ & \textit{Madrid} $\checkmark$ & $0.03$ \\
    \textit{freak folk} & $almost\ 0$ & \textit{Toronto International Film Festival} & $almost\ 0$ \\
    \bottomrule
    \end{tabular}%
  \label{case study table}%
\end{table}%

\section{Comparison with $\text{CQD}^\mathcal{A}$}
\label{comparison with cqda}

The results of $\text{CQD}^{\mathcal{A}}$ compared in our paper are derived from \citep{arakelyan2023adapting}. \citet{arakelyan2023adapting} obtained these results by training on the dataset that includes all ``2i'', ``3i'', ``2in'', and ``3in'' queries. However, our proposed NSMP does not require training on complex queries. To fairly compare $\text{CQD}^{\mathcal{A}}$ and NSMP, we compare $\text{CQD}^{\mathcal{A}}$ using the two data-efficient variants proposed in \citep{arakelyan2023adapting} (i.e., "FB237, 1\%" and "FB237, 2i, 1\%"). These two variants are trained using $1\%$ training data and $1\%$ 2i queries on FB15k-237, respectively, in which case it is relatively fair to compare NSMP, since NSMP has no trainable parameters. The results are shown in Table \ref{cqda}. From these results, in the case where both of these variants still require training data (though very data-efficient), the NSMP, which has no trainable parameters, improves the average MRR score on positive queries by $16.9\%$ and $19.5\%$ respectively compared to these two variants. On negative queries, the average MRR score is improved by $33.3\%$ and $62.2\%$, respectively. Considering such significant performance improvements, we believe that the superiority of NSMP over $\text{CQD}^{\mathcal{A}}$ has been demonstrated.

\begin{table}[htbp]
  \centering
  \caption{MRR results of $\text{CQD}^{\mathcal{A}}$ and its two data-efficient variants on FB15k-237. The average score is calculated separately among positive and negative queries. Highlighted are the top \textbf{first} results.}
  \resizebox{\linewidth}{!}{
    \begin{tabular}{cccccccccccccccc}
    \toprule
    \textbf{Model} & \textbf{1p} & \textbf{2p} & \textbf{3p} & \textbf{2i} & \textbf{3i} & \textbf{pi} & \textbf{ip} & \textbf{2u} & \textbf{up} & \textbf{AVG.(P)} & \textbf{2in} & \textbf{3in} & \textbf{inp} & \textbf{pin} & \textbf{AVG.(N)} \\
    \midrule
    $\text{CQD}^{\mathcal{A}}$ (FB237,2i,3i,2in,3in) & \textbf{46.7} & 13.6  & 11.4  & 34.5  & 48.3  & 27.4  & 20.9  & \textbf{17.6} & 11.4  & 25.7  & \textbf{13.6} & 16.8  & 7.9   & \textbf{8.9} & 11.8 \\
    $\text{CQD}^{\mathcal{A}}$ (FB237,1\%) & \textbf{46.7} & 11.8  & 11.4  & 33.6  & 41.2  & 24.8  & 17.8  & 16.5  & 8.7   & 23.6  & 10.8  & 13.9  & 5.9   & 5.4   & 9.0 \\
    $\text{CQD}^{\mathcal{A}}$ (FB237,2i,1\%) & \textbf{46.7} & 11.8  & 11.2  & 30.4  & 40.8  & 23.4  & 18.3  & 15.9  & 9.0   & 23.1  & 9.4   & 10.3  & 5.2   & 4.5   & 7.4 \\
    NSMP (No trainable Parameters) & \textbf{46.7} & \textbf{15.1} & \textbf{12.3} & \textbf{38.7} & \textbf{52.2} & \textbf{31.2} & \textbf{23.3} & 17.2  & \textbf{11.9} & \textbf{27.6} & 11.9  & \textbf{17.6} & \textbf{10.8} & 7.9   & \textbf{12.0} \\
    \bottomrule
    \end{tabular}%
    }
  \label{cqda}%
\end{table}%

\section{How Much Efficiency Improvement Can Dynamic Pruning Achieve?}
\label{dp efficiency}

The dynamic pruning strategy filters out noisy messages between variable nodes, thereby reducing the computational cost associated with these messages. As discussed in Section \ref{complexity}, the primary computational bottleneck of NSMP lies in the computation of neural-symbolic messages. This indicates that the dynamic pruning strategy effectively alleviates a portion of the computational burden during message passing. Additionally, the dynamic pruning strategy does not update the states of variable nodes that do not receive any messages, further reducing the computation required for state updates. In summary, the dynamic pruning strategy can reduce the state update computation and message computation in the forward process, effectively improving inference efficiency. 
To verify this, we conduct experiments on inference time on an NVIDIA RTX 3090 GPU. We consider the comparison between the original NSMP (i.e., the version with dynamic pruning) and NSMP without dynamic pruning (i.e., NSMP w/o DP) on FB15k-237. Specifically, we measured the average inference time required by the original NSMP and NSMP w/o DP to process each type of test query on FIT datasets. The detailed inference times, expressed in milliseconds per query (ms/query), are provided in Table \ref{dp efficient experiment}. From the results, it can be observed that the dynamic pruning strategy not only effectively enhances model performance (see Section \ref{ablation}) but also improves efficiency. Compared to the variant without dynamic pruning, NSMP achieves a $23.8\%$ reduction in average inference time. We believe that these empirical results validate our view that the dynamic pruning strategy can effectively enhance efficiency.

\begin{table}[htbp]
  \centering
  \caption{Inference time (ms/query) on each query type on FB15k-237.}
  \resizebox{0.85\linewidth}{!}{
    \begin{tabular}{cccccccccccc}
    \toprule
    \textbf{Model} & \textbf{pni} & \textbf{2il} & \textbf{3il} & \textbf{2m} & \textbf{2nm} & \textbf{3mp} & \textbf{3pm} & \textbf{im} & \textbf{3c} & \textbf{3cm} & \textbf{Avg} \\
    \midrule
    NSMP w/o DP & 14.4  & 11.4  & 14.0  & 16.6  & 16.6  & 29.2  & 29.2  & 19.4  & 32.8  & 39.8  & 22.3 \\
    NSMP  & 12.8  & 9.0   & 11.8  & 11.8  & 12.0  & 20.6  & 18.8  & 14.8  & 27.0  & 31.8  & 17.0 \\
    \bottomrule
    \end{tabular}%
    }
  \label{dp efficient experiment}%
\end{table}%


\end{document}